\documentclass[11pt]{article}
\usepackage{amsmath,amssymb,amsthm}
\usepackage{fullpage}
\usepackage[hcentering]{geometry}
\usepackage{caption} \usepackage{subcaption}
\usepackage{graphicx}
\usepackage{float}
\usepackage{placeins}
\usepackage{paralist}
\usepackage[utf8]{inputenc}   
\usepackage[T1]{fontenc}

\newif\iffinal
\finalfalse
\newif\ifarxiv
\arxivfalse

\usepackage{paralist}
\usepackage{url}

\usepackage{aliascnt}
\def\NewTheorem#1#2{%
  \newaliascnt{#1}{theorem}
  \newtheorem{#1}[#1]{#2}
  \aliascntresetthe{#1}
  \expandafter\def\csname #1autorefname\endcsname{#2}
}
 \newtheorem{theorem}{Theorem}[section]
\NewTheorem{lemma}{Lemma}
\NewTheorem{corollary}{Corollary}
\NewTheorem{conjecture}{Conjecture}
\NewTheorem{observation}{Observation}
\NewTheorem{definition}{Definition}
\NewTheorem{assumption}{Assumption}

\NewTheorem{proposition}{Proposition}
\NewTheorem{remark}{Remark}
\NewTheorem{claim}{Claim}
\NewTheorem{fact}{Fact}

\newcommand{\hide}[1]{}

\DeclareMathOperator{\lmax}{\ell_{max}}

\usepackage{color}
\usepackage[usenames,dvipsname s,svgnames,table]{xcolor}

\definecolor{darkred}{rgb}{0.5,0,0}
\definecolor{lightblue}{rgb}{0,0.4,0.8}
\definecolor{darkgreen}{rgb}{0,0.5,0}

\usepackage[colorlinks=true, linkcolor=blue, urlcolor=blue, citecolor=darkgreen]{hyperref}

\begin{document}
\title{Learning Hierarchically-Structured Concepts II:  Overlapping Concepts, and Networks With Feedback}
\author{Nancy Lynch, Frederik Mallmann-Trenn}
	\date{}	
\maketitle 

\begin{abstract}
We continue our study from~\cite{LM21}, of how concepts that have hierarchical structure might be represented in brain-like neural networks, how these representations might be used to recognize the concepts, and how these representations might be learned.
In~\cite{LM21}, we considered simple tree-structured concepts and feed-forward layered networks.
Here we extend the model in two ways:  we allow limited overlap between children of different concepts, and we allow networks to include feedback edges.
For these more general cases, we describe and analyze algorithms for recognition and algorithms for learning.
\end{abstract}

\tableofcontents

\section{Introduction}

We continue our study, begun in~\cite{LM21}, of how concepts that have hierarchical structure might be represented in brain-like neural networks, how these representations might be used to recognize the concepts, and how these representations might be learned.
In~\cite{LM21}, we considered only simple tree-structured concepts and
simple feed-forward layered networks.
Here we consider two important extensions:  we allow our data model to
include limited \emph{overlap} between the sets of children of different concepts,
and we extend the network model to allow some \emph{feedback edges}.
We consider these extensions both separately and together.
In all cases, we consider both algorithms for recognition and algorithms
for learning.
Where we can, we quantify the effects of these extensions on the
costs of recognition and learning algorithms.

In this paper, as in~\cite{LM21}, we consider \emph{robust recognition}, which means that recognition of a concept is guaranteed even in the absence of some of the lowest-level parts of the concept.\footnote{One might also consider what happens in the presence of a small number of extraneous inputs.  We do not address this case here, but discuss this as future work, in Section~\ref{sec: conclusions}.}
In~\cite{LM21}, we considered both noise-free learning and learning in the presence of random noise.
Here we emphasize noise-free learning, but include some ideas for extending the results to the case of noisy learning.

\vspace{-.3cm}
\paragraph{Motivation:}
This work is inspired by the behavior of the visual cortex, and by algorithms used for computer vision.
As described in~\cite{LM21}, we are interested in the general problem of how concepts that have structure are represented in the brain. What do these representations look like? How are they learned, and how do the concepts get recognized after they are learned?  
We draw inspiration from experimental research on computer vision in convolutional neural networks (CNNs) by Zeiler and Fergus~\cite{Zeiler} and Zhou, et al.~\cite{8417924}. This research shows that CNNs learn to represent structure in visual concepts: lower layers of the network represent basic concepts and higher layers represent successively higher-level concepts.
This observation is consistent with neuroscience research, which indicates that visual processing in mammalian brains is performed in a hierarchical way, starting from primitive notions such as position, light level, etc., and building toward complex objects; see, e.g., \cite{Hubel59, Hubel62, Felleman91}.

In~\cite{LM21}, we considered only tree-structured concepts and feed-forward layered networks.  Here we allow overlap between sets of children of different concepts, and feedback edges in the network.
Overlap may be important, for example, in a complicated visual scene in which one object is part of more than one higher-level object, like a corner board being part of two sides of a house.
Feedback is critical in visual recognition, since once we recognize a particular higher-level object, we can often fill in lower-level details that were not easily recognized without the help of the context provided by the object.  For example, once we recognize that we are looking at a dog, based on seeing some of its parts, we can recognize other parts that are less visible, such as a partially-occluded leg.
 
\vspace{-.3cm}
\paragraph{Paper contents:}
We begin in Section~\ref{sec:datamodel} by extending our formal concept hierarchy model of~\cite{LM21}.  The only change is the allowance of limited amounts of overlap in the sets of children of concepts at the same level in the hierarchy.  
We define two notions of $support$, one involving only bottom-up information flow as in~\cite{LM21}, and the other also allowing top-down information flow.
The recognition algorithms presented later in the paper will closely follow these definitions.

We continue in Section~\ref{sec: networkmodel} with definitions of our networks, both feed-forward and with feedback.  
Feed-forward networks are as in~\cite{LM21}, with only ``upward'' edges from children to parents.
Networks with feedback add corresponding ``downward'' edges from parents to children.
Incoming potential for a neuron is calculated based on all incoming edges, both upward and downward; all incoming edges are treated in the same way.
However, for learning weights on edges, we consider different rules for upward and downward edges.
Next, in Section~\ref{sec: probstatement}, we define the robust recognition and noise-free learning problems.
% [[[We define the noisy learning problem in the learning section, where we discuss the solutions informally.]]]

Section~\ref{sec: recognition-ff} contains algorithms for robust recognition in feed-forward networks, for both tree hierarchies and general hierarchies.
We start with basic recognition results, for a setting in which weights are either $1$ or $0$ and the firing threshold has a simple form; for this setting, we obtain a precise characterization of which neurons fire at which times, which leads to a robust recognition theorem.
%We give a sufficient condition for avoiding anomalous firing caused by overlap.
%
We describe how the recognition results extend to settings in which the weights are known only approximately, and to settings in which the weights and thresholds are uniformly scaled.
Finally, we consider how the results change if the neurons' firing decisions are made stochastically, rather than being determined by a fixed threshold.

Section~\ref{sec: recognition-feedback} contains algorithms for robust recognition in networks with feedback, for both tree hierarchies and general hierarchies.
For tree hierarchies in networks with feedback, we show that recognition requires only enough time for two passes:  an upward pass to recognize whatever can be recognized based on lower-level information, and a downward pass to recognize concepts based on a combination of higher-level and lower-level information. 
On the other hand, for general concept hierarchies in networks with feedback, we may need much more time---enough for many passes, both upward and downward.
We again describe how the recognition results extend to settings in which weights are approximate, and in which weights and thresholds are scaled.

In Section~\ref{sec: learning-ff}, we describe noise-free learning algorithms in feed-forward networks, which produce edge weights for the upward edges that suffice to support robust recognition.
These learning algorithms are adapted from the noise-free learning algorithm in~\cite{LM21}, and work for both tree hierarchies and general concept hierarchies.
The extension to general hierarchies requires us to reconsider the use of the \emph{Winner-Take-All (WTA)} module in the algorithm of~\cite{LM21}, since the previous version does not work with overlap; we present a new version of the WTA mechanism.
As before, the weight adjustments are based on Oja's rule~\cite{Oja}.
We show that our new learning algorithms can be viewed as producing approximate, scaled weights as described in Section~\ref{sec: recognition-ff}, which can be used to decompose the correctness proof for the learning algorithms.
Finally, we briefly discuss extensions to noisy learning.

In Section~\ref{sec: learning-feedback}, we extend the learning algorithms for feed-forward networks to accommodate feedback.
Here we simply separate the learning mechanisms for the weights of the upward and downward edges, using a different rule for learning the weights of the downward edges.  Our learning mechanism for the upward weights is based on Oja's rule, whereas learning for downward weights uses a simpler, all-at-once, Hebbian-style rule.\footnote{This may seem a bit inconsistent.  The main reason to use Oja's incremental rule for learning the upward weights is to tolerate noise during the learning process.  We are not emphasizing noisy learning in the paper, but we do expect that the results will extend to that case. So far we have not thought about noise for learning the downward weights, so we use the simplest option, which is an all-at-once rule.}

Section~\ref{sec: conclusions} contains our conclusions.

\vspace{-.3cm}
\paragraph{Contributions:}
We think that the most interesting contributions in the paper are:
\begin{enumerate}
    \item  Formal definitions for concept hierarchies with overlap, and networks with feedback (Sections~\ref{sec: concept-hierarchies} and~\ref{sec: networks-feedback-defs}).
    \item  Analysis for time requirements for robust recognition; this is especially interesting for general concept hierarchies in networks with feedback (Sections~\ref{sec: recognition-ff-exact}, ~\ref{sec: time-bounds-trees-feedback}, and~\ref{sec: time-bounds-recognition-overlap-feedback}). 
    \item  Extensions of the recognition results to allow approximate edge weights and scaling (Sections~\ref{sec: uncertain-weights1}, \ref{sec: scaled}, and~\ref{sec: uncertain-weights-2}).
    \item  The handling of Winner-Take-All mechanisms during learning, in the presence of overlap (Section~\ref{sec: learning-trees} and~\ref{sec: learning-general}).
    \item  Reformulation of learning behavior in terms of achieving certain ranges of edge weights (Sections~\ref{sec: learning-weight-ranges-3} and~\ref{sec: noisy-learning}).
    \item  A simple mechanism for learning bidirectional weights (Section~\ref{sec: learning-feedback}).
\end{enumerate}  

%%%%%%%%%%%%%%%%%%%%%%%%%%%%%%%%%%%%%%%%%%%%%%
\section{Data Models}
\label{sec:datamodel}

We use two types of data models in this paper.  One is the same type
of tree hierarchy as in~\cite{LM21}.
The other allows limited overlap in the sets of children of different concepts. 
As before, a concept hierarchy is supposed to represent all the
concepts that are learned in the ``lifetime'' of an organism, together with parent/child
relationships between them.

We also include two definitions for the notion of ``supported'', which are used to describe the set of concepts whose recognition should be triggered by a given set of basic concepts.
One definition is for the case where information flows only upwards, from children to parents, while the other also allows downward flow, from parents to children.
These definitions capture the idea that recognition is robust, in the sense that a certain fraction of neighboring (child and parent) concepts should be enough to support recognition of a given concept.

\subsection{Preliminaries}
\label{sec: prelims}

We start by defining some parameters:
%\footnote{We might want to regard these as parameters rather than constants; for example, we might want to compare networks that guarantee robustness for different values of $r_1$ and $r_2$.}:
\begin{itemize}
    \item $\lmax$, a positive integer, representing the maximum level
      number for the concepts we consider,
    \item $n$, a positive integer, representing the total number of
      lowest-level concepts,
    \item $k$, a positive integer, representing the number of top-level concepts in a concept hierarchy, and the number of sub-concepts for each concept that is not at the lowest level in the hierarchy, 
    \item $r_1, r_2$, reals in $[0,1]$ with $r_1 \leq r_2$; these
      represent thresholds for robust recognition,
    \item $o$, representing an upper bound on overlap, and
    \item $f$, a nonnegative real, representing strength of feedback.
\end{itemize}

We assume a universal set $D$ of \emph{concepts}, partitioned into
disjoint sets $D_{\ell}, 0 \leq \ell \leq \lmax$.
We refer to any particular concept $c \in D_{\ell}$ as a \emph{level}
$\ell$ \emph{concept}, and write $level(c) = \ell$.
Here, $D_0$ represents the most basic concepts and $D_{\lmax}$ the
highest-level concepts.
We assume that $|D_0| = n$.

\subsection{Concept Hierarchies}
\label{sec: concept-hierarchies}

We define a general notion of a concept hierarchy, which allows overlap.
We will refer to our previous notion from~\cite{LM21} as a ``tree concept hierarchy'' ; it can be defined by a simple restriction on the general definition.

A \emph{(general) concept hierarchy} $\mathcal C$ consists of a subset
$C \subseteq D$, together with a $children$ function, satisfying the constraints below.
For each $\ell$, $0 \leq \ell \leq \lmax$, we define $C_{\ell}$ to be
$C \cap D_{\ell}$, that is, the set of level $\ell$ concepts in
$\mathcal C$.
For each concept $c \in C_0$, we assume that $children(c) = \emptyset$.
For each concept $c \in C_{\ell}$, $1 \leq \ell \leq \lmax$, we assume that $children(c)$ is a nonempty subset of $C_{\ell-1}$.
We call each element of $children(c)$ a \emph{child} of $c$.

We extend the $children$ notation recursively, namely, we define
concept $c'$ to be a $descendant$ of a concept $c$ if either $c' = c$,
or $c'$ is a child of a descendant of $c$.
We write $descendants(c)$ for the set of descendants of $c$.
Let $leaves(c) = descendants(c) \cap C_0$, that is, the set of all level $0$
descendants of $c$.

Also, we call every concept $c'$ for which $c \in children(c')$ a \emph{parent} of $c$, and write $parents(c)$ for the set of parents of $c$.
Since we allow overlap, the set $parents(c)$ might contain more than one element.
If a concept $c$ has only one parent, we write $parent(c)$.

We assume the following properties:
\begin{enumerate}
\item  
$|C_{\lmax}| = k$.
That is, the number of top-level concepts in the hierarchy is exactly $k$.\footnote{This assumption and the next are just for uniformity, to reduce the number of parameters and simplify the math.}
\item 
For any $c \in C_{\ell}$, where $1 \leq \ell \leq \lmax$, we have that $|children(c)| = k$. That is, the number of children of any non-leaf concept is exactly $k$.
\item 
\emph{Limited overlap:}
Let $c \in C_\ell$, where $1 \leq \ell \leq \lmax$.
Let $C' = \bigcup_{c' \in C_{\ell} - \{ c \}} children(c')$; that is, $C'$ is the
union of the sets of children of all the other concepts in $C_\ell$, other than $c$.
Then $|children(c) \cap C'| \leq o \cdot k$.
\end{enumerate}

To define a \emph{tree hierarchy}, we replace the limited overlap property with the stronger property:

\begin{enumerate}
\item [4.]
\emph{No overlap:}
For any two distinct concepts $c$ and $c'$ in $C_{\ell}$, where $1 \leq \ell \leq \lmax$,
we have that $children(c) \cap children(c') = \emptyset$.  That is, the
sets of children of different concepts at the same level are disjoint.
This property is equivalent to Property 3 with $o = 0$. 
\end{enumerate}

Properties 1, 2, and 4 are the same as in~\cite{LM21}.
Property 3 is new here:  we replace the no-overlap condition assumed in~\cite{LM21} with a
condition that limits the overlap between the set of children of a
concept $c$ and the sets of children of all other concepts at the same level.
We require this overlap to be less than a designated fraction $o \cdot k$ of the children of $c$.

% [[[Earlier I defined a non-overlap condition in terms of pairwise
% intersection, but that was not enough.
% So here I consider the intersection of the children of one concept
% with the children of all the other concepts taken together. This
% reflects the possibility that, during recognition many such concepts
% could have their children firing at the same time; we would like to
% avoid having this possibility cause a spurious firing of $rep(c)$.]]]

In Appendix~\ref{app: restaurant}, we give a simple example of a concept hierarchy.  The example is not based on scene recognition, which was the main motivation for this work.  Instead, it describes a much simpler structure:  a catering menu for an Italian restaurant.  The menu consists of meals, which in turn consist of dishes, which in turn consist of ingredients.  There is limited overlap between the ingredients in different dishes.

\subsection{Support}
\label{sec:support}

%[[[Think about what might change here as a result of adding some "extraneous" inputs.]]]

In this subsection, we fix a particular concept hierarchy ${\mathcal C}$, with its concept set $C$, partitioned into $C_0, C_1, \ldots, C_{\lmax}$.  We assume that $\mathcal C$ satisfies the limited-overlap property.
We give two definitions, one that expresses only upward information flow and one that also expresses downward information flow.

Both definitions are illustrated in Appendix~\ref{app: restaurant}.

\subsubsection{Support with only upward information flow}

For any given subset $B$ of the general set $D_0$ of level $0$
concepts, and any real number $r \in [0,1]$, we define a set
$supp_r(B)$ of concepts in $C$.
This represents the set of concepts $c \in C$, at any level, such that $B$ contains enough leaves of $c$ to support recognition of $c$.
The notion of ``enough'' here is defined recursively, in terms of a level parameter $\ell$.
%based on having at least an $r$-fraction of children supported for every descendant, at every level.
This definition is equivalent to the corresponding one in~\cite{LM21}.

\begin{definition}
\label{def: supp1}
$supp_r(B)$:  Given $B \subseteq D_0$, define the sets of concepts $S(\ell)$, for $0 \leq \ell \leq \lmax$:
\begin{itemize}
    \item $S(0) = B \cap C_0$.
    \item For $1 \leq \ell \leq \ell_{max}$, $S(\ell)$ is
      the set of all concepts $c \in C_{\ell}$ such that $|children(c)
      \cap S(\ell - 1)|  \geq r k$.
\end{itemize}
Define $supp_r(B)$ to be $\bigcup_{0 \leq \ell \leq \lmax} S(\ell)$.  
We also write $supp_r(B,\ell)$ for $S(\ell)$, when we want to make the parameters $r$ and $B$ explicit.
\end{definition}

The following monotonicity lemma says that increasing the value of the parameter $r$ can only decrease the supported set.\footnote{The mention of the limited-overlap property is just for emphasis, since all of the concept hierarchies of this paper satisty this property.}

\begin{lemma}
\label{lem: support-monotonic}
Let $\mathcal C$ be any concept hierarchy satisfying the limited-overlap property, and let $B \subseteq D_0$.
Consider $r, r'$, where $0 \leq r \leq r' \leq 1$.
Then $supp_{r'}(B) \subseteq supp_r(B)$. 
\end{lemma}

The following lemma says that any concept $c$ is supported by its entire set of leaves.

\begin{lemma}
\label{lem: total-support}
Let $\mathcal C$ be any concept hierarchy satisfying the limited-overlap property.
If $c$ is any concept in $C$, then $c \in supp_{1} (leaves(c))$.
\end{lemma}

\begin{proof}
By induction on $level(c)$.
\end{proof}

\subsubsection{Support with both upward and downward information flow}

Our second definition, which captures information flow both upward and downward in the concept hierarchy, is a bit more complicated.  It is expressed in terms of a generic ``time parameter''  $t$, in addition to the level parameter $\ell$.
Here, $f$ is a nonnegative real, as specified at the start of Section~\ref{sec: prelims}.

\begin{definition}
\label{def: supp2}
$supp_{r,f}(B)$:  Given $B \subseteq D_0$, define the sets of concepts $S(\ell,t)$, for $0 \leq \ell \leq \lmax$ and $t \geq 0$:
\begin{enumerate}
    \item $\ell = 0$ and $t \geq 0$: 
    Define $S(0,t) = B$.  $B$ is initially supported and continues to be supported, and no level $0$ concept other than those in $B$ ever gets supported.
    \item $1 \leq \ell \leq \lmax$ and $t = 0$:
    Define $S(\ell,0) = \emptyset$.  No concepts at levels higher than $0$ are initially supported.
    \item $1 \leq \ell \leq \lmax$ and $t \geq 1$:
        Define 
  \[S(\ell,t) = S(\ell,t-1) \ \cup \
  \{c \in C_{\ell} : |children(c) \ \cap \ S(\ell-1,t-1)| 
  \ + \ f \ |parents(c) \ \cap \ S(\ell+1,t-1)| \geq r k \}.
  \]
  Thus, concepts that are supported at time $t-1$ continue to be supported at time $t$.  In addition, new level $\ell$ concepts can get supported at time $t$ based on a combination of children and parents being supported at time $t-1$, with a weighting factor $f$ used for parents.
\end{enumerate}
Define $supp_{r,f}(B)$ to be $\bigcup_{\ell,t} S(\ell,t)$.
%For each $\ell$, define $B_{\ell}$ to be $\bigcup_{t} B_{\ell,t}$.
We sometimes also write $supp_{r,f}(B,\ell,t)$ for $S(\ell,t)$, when we want to make the parameters $r$, $f$, and $B$ explicit.
%and $supp_{r,f}(B,\ell)$ for $B_{\ell}$.
\end{definition}

We also use the abbreviations $supp_{r,f}(B,*,t)$ for $\bigcup_{\ell}S(\ell,t)$,
$supp_{r,f}(B,\ell,*)$ for $\bigcup_{t}S(\ell,t)$, and 
$supp_{r,f}(B,*,*)$ for $\bigcup_{\ell,t}S(\ell,t)$,
Notice that each of these three unions must converge to a finite set 
since all the sets $S(\ell,t)$ are subsets of the single finite set
$C_{\ell}$ of concepts.

Now we have two monotonicity results, for $r$ and $f$:

\begin{lemma}
\label{lem: support-monotonic-2}
Let $\mathcal C$ be any concept hierarchy satisfying the limited-overlap property, and let $B \subseteq D_0$.
Consider $r, r'$, where $0 \leq r \leq r' \leq 1$, and arbitrary $f$.
Then $supp_{r',f}(B) \subseteq supp_{r,f}(B)$. 
\end{lemma}

\begin{lemma}
\label{lem: support-monotonic-3}
Let $\mathcal C$ be any concept hierarchy satisfying the limited-overlap property, and let $B \subseteq D_0$.
Consider $f,f'$, where $0 \leq f \leq f'$, and arbitrary $r \in [0,1]$.
Then $supp_{r,f}(B) \subseteq supp_{r,f'}(B)$. 
\end{lemma}

Also note that the second $supp$ definition with $f = 0$ corresponds to the first definition:

\begin{lemma}
Let $\mathcal C$ be any concept hierarchy satisfying the limited-overlap property, and $B \subseteq D_0$.
Then $supp_{r,0}(B) = supp_{r}(B)$.
Moreover, for every $\ell$, $0 \leq \ell \leq \lmax$, 
$supp_{r,f}(B,\ell,*) = supp_r(B,\ell)$.
%that is, the auxiliary $S_{\ell}$ sets are the same in the two definitions.
\end{lemma}

\subsubsection{Time bounds}
\label{sec: time-bounds}

We would like upper bounds on the time by which the sets $S(\ell,t)$ in the second $supp$ definition stabilize to their final values.  Specifically,  for each value of $\ell$, we would like to find a value $t^*$ such that $S(\ell,t^*) = supp_{r,f}(B,\ell,*)$.  It follows that, for every $t \geq t^*$, $S(\ell,t) = supp_{r,f}(B,\ell,*)$.

In general, we have only a large (exponential in $\lmax$) upper bound, based on the fact that $C$ contains only a bounded number of concepts.
However, we have better results in two special cases.
The first result is for the case where $f = 0$, that is, where there is no feedback from parents.  In this case, for every $\ell$, the sets $S(\ell,t)$ stabilize within time $\ell$, as the support simply propagates upwards.

\begin{theorem}
\label{lem: tree-upward-support}
Let $\mathcal C$ be any concept hierarchy satisfying the limited-overlap property, and let $B \subseteq D_0$.
Then for any $\ell$, $0 \leq \ell \leq \lmax$, we have
$supp_{r,0}(B,\ell,*) = supp_{r,0}(B,\ell,\ell)$.
\end{theorem}

\begin{proof}
By induction on $\ell$.
\end{proof}

It follows that $supp_{r,0}(B) = \bigcup_{\ell} supp_{r,0}(B,\ell,\ell)$.
Since $\lmax$ is the maximum value of $\ell$, we get that
$supp_{r,0}(B) = \bigcup_{\ell} supp_{r,0}(B,\ell,\lmax) = supp_{r,0}(B,*,\lmax)$.
This means that all the sets $S(\ell,t)$ stabilize to their final values by time $t = \lmax$.

The second result is for the special case of a tree hierarchy, with any value of $f$. In this case, the support may propagate both upwards and downwards.  This propagation may follow a complicated schedule, but the total time is bounded by $2 \lmax$. 

To prove this, we use a lemma saying that, if a concept $c$ gets put into an $S(\ell,t)$ set before its parent does, then $c$ is supported by just its descendants.
To state the lemma, we here abbreviate $supp_{r,f}(B,*,t)$ by simply $S(t)$.
Thus, $S(t)$ is the set of concepts at all levels that are supported by input set $B$, by step $t$ of the recursive definition of the $S(\ell,t)$ sets.
The lemma says that, if a concept $c$ is in $S(t)$ and its parent is not, then it must be that $c$ is supported by just its descendants:

\begin{lemma}
\label{lem: node-and-parent}
Let $\mathcal C$ be any tree concept hierarchy, and let $B \subseteq D_0$.
Let $t$ be any nonnegative integer.
If $c \in S(t)$ and $parent(c) \notin S(t)$ then $c \in supp_{r,0}(B)$.
\end{lemma}

\begin{proof}
We proceed by induction on $t$.
The base case is $t=0$.  In this case $c \in S(0)$, which implies that
$c \in supp_{r,0}(B)$, as needed.

For the inductive step, assume that $c \in S(t+1)$ and $parent(c) \notin S(t+1)$.
If $c \in S(t)$, then since $parent(c) \notin S(t)$, the inductive hypothesis tells us that $c \in supp_{r,0}(B)$.
So assume that $c \notin S(t)$.
Then since $parent(c) \notin S(t)$, $c$ is included in $S(t+1)$ based only on its children who are in $S(t)$.
Since $c \notin S(t)$, the parent of each of these children is not in $S(t)$.
Then by inductive hypothesis, all of $c$'s children that are in $S(t)$ are in $supp_{r,0}(B)$.
Since they are enough to support $c$'s inclusion in $supp_{r,0}(B)$, we have that $c \in supp_{r,0}(B)$.
\end{proof}

Now we can state our time bound result.  It says that, for the case of tree hierarchies, the sets $S(\ell,t)$ stabilize within time $2 \lmax - \ell$.

\begin{theorem}
\label{lem: support-firing}
Let $\mathcal C$ be any tree concept hierarchy, and let $B \subseteq D_0$.
Then for any $\ell$, $0 \leq \ell \leq \lmax$, we have 
$supp_{r,f}(B,\ell,*) = supp_{r,f}(B,\ell,2\lmax-\ell)$.
%In terms of the auxiliary sets, we have $B_\ell = B_{\ell,2 \lmax}$.
\end{theorem}

\begin{proof}
This theorem works because, for each $\ell$, the $S(\ell,t)$ sets stabilize after support has had a chance to propagate upwards from level $0$ to level $\lmax$, and then propagate downwards from level $\lmax$ to level $\ell$.  Because the concept hierarchy has a simple tree structure, there is no other way for a concept to get added to the $S(\ell,t)$ sets.

Formally, we use a backwards induction on $\ell$, from $\ell = \lmax$ to $\ell = 0$, to prove the claim:  If $c \in supp_{r,f}(B)$ and $level(c) = \ell$, then $c \in S(2 \lmax - \ell)$.
For the base case, consider $\ell = \lmax$.
Since $c$ has no parents, we must have $c \in supp_{r,0}(B)$, so 
Theorem~\ref{lem: tree-upward-support} implies that $c \in S(\lmax) \subseteq S(2 \lmax)$, as needed.

For the inductive step, suppose that $c \in supp_{r,f}(B)$ and $level(c) = \ell-1$.
If $c \in supp_{r,0}(B)$ then again Theorem~\ref{lem: tree-upward-support} implies that $c \in S(\lmax)$, which suffices.
So suppose that $c \in supp_{r,f}(B) - supp_{r,0}(B)$.
Then $c$'s inclusion in the set $supp_{r,f}(B)$ relies on its (unique) parent first being included, that is, there is some $t$ for which $c \notin S(t)$ and $parent(c) \in S(t)$.
Since $parent(c) \in supp_{r,f}(B)$ and $level(parent(c)) = \ell$, the inductive hypothesis yields that $parent(c) \in S(2 \lmax - \ell)$.

Moreover, all the children that $c$ relies on for its inclusion in $supp_{r,f}(B)$ are in $supp_{r,0}(B)$, by Lemma~\ref{lem: node-and-parent}.
Therefore, by Theorem~\ref{lem: tree-upward-support}, they are in $S(\ell-2) \subseteq S(2 \lmax - \ell)$.
Thus, we have enough support for $c$ in $S(2 \lmax - \ell)$ to ensure that $c \in S(2 \lmax - \ell + 1)$, as needed.
\end{proof}

%%%%%%%%%%%%%%%%%%%%%%%%%%%%%%%%%%%%%%%%%%%%%%%%%%%%%%%%%%%%%%%%%%%%%%%%%%%%%%%%%%%%%%
\section{Network Models}
\label{sec: networkmodel}

We consider two types of network models, for feed-forward networks and for networks with feedback.
The model for feed-forward networks is the same as the network model in~\cite{LM21}, with upward edges between consecutive layers.
The model for networks with feedback includes edges in both directions between consecutive layers.

\subsection{Preliminaries}

We define the following parameters:
\begin{itemize}
    \item $\ell'_{max}$, a positive integer, representing the maximum
      number of a layer in the network,
    \item $n$, a positive integer, representing the number of distinct
      inputs the network can handle; this is intended to match up with the parameter $n$ in the data model,
    \item $f$, a nonnegative real, representing strength of feedback; this is intended to match up with the parameter $f$ in the data model,
 % [[[This will be used in defining larger weights for edges.]]]
    \item $\tau$, a real number, representing the firing threshold for neurons, and
    \item $\eta$, a positive real, representing the learning rate for Oja's rule.
\end{itemize}

\subsection{Network Structure}

A network $\mathcal{N}$ consists of a set $N$ of neurons, partitioned into disjoint sets $N_{\ell}, 0 \leq \ell \leq \ell'_{max}$, which we call \emph{layers}.
We assume that each layer contains exactly $n$ neurons, that is, $|N_\ell|= n$ for every $\ell$.
We refer to the $n$ neurons in layer $0$ as \emph{input neurons}.
%and to all other neurons as \emph{non-input neurons}.

For feed-forward networks, we assume that each neuron in
$N_{\ell}$, $0 \leq \ell \leq \ell'_{max} - 1$, has an outgoing ``upward''
edge to each neuron in  $N_{\ell+1}$, and that these are the only edges in the network.
For networks with feedback, we assume that, in addition to these upward edges, each neuron in $N_{\ell}$, $1 \leq \ell \leq \ell'_{max}$, has an outgoing ``downward'' edge to
each neuron in  $N_{\ell-1}$.

We assume a one-to-one mapping $rep : D_0 \rightarrow N_0$, where
$rep(c)$ is the input neuron corresponding to level $0$ concept $c$.
That is, $rep$ is a mapping from the full set of level $0$ concepts to
the full set of layer $0$ neurons.
This allows the network to receive an input corresponding to any level $0$ concept, using a simple unary encoding.

\subsection{Feed-Forward Networks}

Here we describe the contents of neuron states and the rules for network operation, for feed-forward networks.

\subsubsection{Neuron states}

Each input neuron $u \in N_0$ has just one state component:
\emph{firing}, with values in $\{0,1\}$; $firing = 1$ indicates that the neuron is firing, and $firing = 0$ indicates that it is not firing.
We denote the \emph{firing} component of neuron $u$ at time $t$ by
$firing^u(t)$.
We assume that the value of $firing^u(t)$, for every time $t$, is set by some external input signal and not by the network.

Each non-input neuron $u \in N_{\ell}$, $1 \leq \ell \leq
\ell'_{max}$, has three state components:
\begin{itemize}
\item \emph{firing}, with values in $\{0,1\}$,
\item \emph{weight}, a length $n$ vector with entries that are reals in the interval $[0,1]$,
and 
\item\emph{engaged}, with values in $\{0,1\}$.
\end{itemize}
The \emph{weight} component keeps track of the weights of incoming edges to $u$ from all neurons at the previous layer.
The \emph{engaged} component is used to indicate whether neuron $u$ is currently prepared to learn new weights.
We denote the three components of non-input neuron $u$ at time $t$ by
$firing^u(t)$, $weight^u(t)$, and $engaged^u(t)$, respectively.
The initial values of these components are specified as part of an
algorithm description.
The later values are determined by the operation of the network, as described below.

% [[[Make sure we do this for our recognition and learning algorithms.]]]

\subsubsection{Potential and firing}

Now we describe how to determine the values of the state components of any
non-input neuron $u$ at time $t \geq 1$, based on $u$'s state
and on firing patterns for its incoming neighbors at the previous time $t-1$.

In general, let $x^u(t)$ denote the column vector of
\emph{firing} values of $u$'s incoming neighbor neurons at time $t$.
Then the value of $firing^u(t)$ is determined by neuron $u$'s
\emph{potential} for time $t$ and its \emph{activation function}.
Neuron $u$'s potential for time $t$, $pot^u(t)$, is given by the dot product of
its $weight$ vector, $weight^u(t-1)$, and its input vector, $x^u(t-1$, at time $t-1$:
\[pot^u(t) = \sum_{j=1}^n \ weight^u_j(t-1) \ x^u_j(t-1).\]
The activation function, which determines whether or not neuron $u$ fires
at time $t$, is then defined by:
\[ firing^u(t) =  \begin{cases}
1 & \text{if $pot^u(t) \geq \tau$}, \\
0 & \text{otherwise}.
\end{cases}\]
Here, $\tau$ is the assumed firing threshold.

\subsubsection{Edge weight modifications}

We assume that the value of the $engaged$ flag of $u$ is controlled externally; that is, for every $t$, the value of $engaged^u(t)$ is set by an external input signal.\footnote{This is a departure from our usual models~\cite{LM21, LMP19, LynchMusco-arxiv21}, in which the network determines all the values of the state components for non-input neurons.  We expect that the network could be modeled as a composition of sub-networks.  One of the sub-networks---a \emph{Winner-Take-All (WTA)} sub-network--- would be responsible for setting the $engaged$ state components.  We will discuss the behavior of the WTA sub-network in Section~\ref{sec: learning-ff}, but a complete compositional model remains to be developed.}

We assume that each neuron that is engaged
at time $t$ determines its weights at time $t$ according to Oja's
learning rule~\cite{Oja}.
That is, if $engaged^u(t) = 1$, then (using vector notation for $weight^u$ and $x^u$):
\begin{equation}\label{eq:Oja} 
\text{\emph{Oja's rule}:
   $weight^u(t) = weight^u(t-1) + \eta \ pot^u(t)\ (x^u(t-1) -   pot^u(t) \ weight^u(t-1) )$.}
\end{equation}
Here, $\eta$ is the assumed learning rate.

\subsubsection{Network operation}

During operation, the network proceeds through a series of
\emph{configurations}, each of which specifies a state for every
neuron in the network.
As described above, the $firing$ values for the input neurons and the $engaged$
values for the non-input neurons are determined by external
sources.  The other state components, which are the $firing$ and $weight$ values
for the non-input neurons, are determined by the initial network specification at
time $t = 0$, and by the activation and learning functions described
above for all times $t > 0$.

%%%%%%%%%%%%%%%%%%%%%%%%%%%%%%%%%%%%%%%%
\subsection{Networks with Feedback}
\label{sec: networks-feedback-defs}

Now we describe the neuron states and rules for network operation for networks with feedback.

\subsubsection{Neuron states}

Each neuron in the network has a state component $firing$, with values in $\{0,1\}$.
In addition, each non-input neuron $u \in N_{\ell}$, $1 \leq \ell < \ell'_{max}$, has state components:
\begin{itemize}
\item \emph{uweight}, a length $n$ vector with entries that are reals in the interval $[0,1]$; these represent weights on ``upward'' edges, i.e., those from incoming neighbors at level $\ell-1$, and
\item\emph{ugaged}, with values in $\{0,1\}$, representing whether $u$ is engaged for learning of $uweights$.
\end{itemize}
And each neuron $u \in N_{\ell}$, $0 \leq \ell \leq \ell'_{max}-1$, has state components:
\begin{itemize}
\item \emph{dweight}, a length $n$ vector with entries that are reals in the interval $[0,f]$; these represent weights on ``downward'' edges, i.e., those from incoming neighbors at level $\ell+1$,
and 
\item\emph{dgaged}, with values in $\{0,1\}$, representing whether $u$ is engaged for learning of $dweights$.
\end{itemize}
We denote the components of neuron $u$ at time $t$ by
$firing^u(t)$, $uweight^u(t)$, $ugaged^u(t)$, $dweight^u(t)$, and $dgaged^u(t)$. 
As before, the initial values of these components are specified as part of an
algorithm description, and the later values are determined by the operation of the network.

\subsubsection{Potential and firing}

For a neuron $u$ at level $\ell$, $1 \leq \ell \leq \ell'_{max} - 1$, the values of the state components of $u$ at time $t \geq 1$ are determined as follows.

In general, let $ux^u(t)$ denote the vector of $firing$ values of $u$'s incoming layer $\ell-1$ neighbor neurons at time $t$, and let $dx^u(t)$ denote the vector of $firing$ values of $u$'s incoming layer $\ell+1$ neighbor neurons at time $t$.
Then, as before, the value of $firing^u(t)$ is determined by neuron $u$'s potential for time $t$ and its activation function.
The potential at time $t$ is now a sum of two potentials, $upot^u(t)$ coming from layer $\ell-1$ neurons and $dpot^u(t)$ coming from layer $\ell+1$ neurons.
We define 
%\[upweight^u(t-1)^T \cdot upx^u(t-1) = \sum_{j=1}^n upweight^u_j(t-1) upx^u_j(t-1).\]
\[upot^u(t) = \sum_{j=1}^n \ uweight^u_j(t-1) \ ux^u_j(t-1),\]
%\[downweight^u(t-1)^T \cdot downx^u(t-1) = \sum_{j=1}^n downweight^u_j(t-1) downx^u_j(t-1).\]
\[dpot^u(t) = \sum_{j=1}^n \ dweight^u_j(t-1) \ dx^u_j(t-1) \mbox{ and }\]
\[pot^u(t) = upot^u(t) + dpot^u(t).\]
The activation function is then defined by:
\[ firing^u(t) =  \begin{cases}
1 & \text{if $pot^u(t) \geq \tau$}, \\
0 & \text{otherwise}.
\end{cases}\]

For a neuron $u$ at level $\ell'_{max}$, the values of the state components of $u$ at time $t \geq 1$ are determined similarly, but using only $uweights$ and $ux$.

\subsubsection{Edge weight modifications}

We assume that the values of the $ugaged$ and $dgaged$ flags of $u$ are controlled externally; that is, for every $t$, the values of $ugaged^u(t)$ and $dgaged^u(t)$ are set by an external input signal.

For updating the weights, we will use two different rules, one for the $uweights$ and one for the $dweights$.
The $uweights$ are modified as before, using Oja's rule based on the previous $uweights$, the $upot$, and the firing pattern of the layer $\ell-1$ neurons.
Specifically, if $ugaged^u(t) = 1$, then
\begin{equation}
\label{eq: Oja2} 
   uweight^u(t) = uweight^u(t-1) + \eta \ (upot^u(t)) \ (ux^u(t-1) -   upot^u(t) \  uweight^u(t-1) ),
\end{equation}
where $\eta$ is the learning rate.

For the $dweights$, we will use a different Hebbian-style learning rule, which we describe in Section~\ref{sec: learning-feedback}.

\subsubsection{Network operation}

During operation, the network proceeds through a series of \emph{configurations}, each of which specifies a state for every neuron in the network.
As before, the configurations are determined by the initial network specification for time $t=0$, and the activation and learning functions.

%%%%%%%%%%%%%%%%%%%%%%%%%%%%%%%%%%%%%%%%%%%%%%
\section{Problem Statements}
\label{sec: probstatement}

In this section, we define the problems we will consider in the rest of this paper---problems of concept recognition and concept learning.
Throughout this section, we use the notation for a concept hierarchy and a network that we defined in Sections~\ref{sec:datamodel} and~\ref{sec: networkmodel}. 
We assume that the concept hierarchy satisfies the limited-overlap property.
We consider both feed-forward networks and networks with feedback,
but the notation we specify here is common to both.

Thus, we consider a concept hierarchy $\mathcal C$, with
concept set $C$ and maximum level $\lmax$, partitioned into $C_0, C_1,
\ldots, C_{\lmax}$.  We use parameters $n$, $k$, $r_1$, $r_2$, $o$, and $f$,
according to the definitions for a concept hierarchy in Section~\ref{sec: concept-hierarchies}.
We also consider a network $\mathcal N$, with maximum layer $\ell'_{max}$, and parameters  $n$, $f$, $\tau$, and $\eta$ as in the definitions for a network in Section~\ref{sec: networkmodel}.
The maximum layer number $\ell'_{max}$ for $\mathcal N$ may be
different from the maximum level number $\lmax$ for $\mathcal C$, but
the number $n$ of input neurons is the same as the number of level $0$
items in $\mathcal C$, and the feedback strength $f$ is the same for both $\mathcal C$ and $\mathcal N$.

We begin with a definition describing how a particular set of level $0$ concepts is ``presented'' to the network.  This involves firing exactly the input neurons that represent these level $0$ concepts.

\begin{definition}
\label{def:presented}
{\bf Presented:}
If $B \subseteq D_0$ and $t$ is a non-negative integer, then we say
that $B$ is \emph{presented at time} $t$ (in some particular network
execution) exactly if the following holds.
For every layer $0$ neuron $u$:
\begin{enumerate}
    \item  If $u$ is of the form $rep(b)$ for some $b \in B$, then $u$ fires at time $t$.
    \item  If $u$ is not of this form, for any $b$, then $u$ does not fire at time $t$.
\end{enumerate}
\end{definition}

%%%%%%%%%%%%%%%%%%%%%%%%%%%%%%%%%%%%%%%%%%
\subsection{Recognition Problems}
\label{sec:prob-recog}

%[[[So far, we are considering only the case where we present level 0 concepts that are within the set of descendants of the concept of interest.  And we want to fire in the "noisy" case where we don't have all the descendants.  Now it seems natural to me that we should also consider the "noisy" case where we have a few extraneous inputs, which are not descendants.  This will involve bounds on the weights of the edges from the non-descendants.]]]

Here we define what it means for network $\mathcal N$ to recognize
concept hierarchy $\mathcal C$.
We assume that every concept $c \in C$, at every level, has a unique
representing neuron $rep(c)$.\footnote{Real biological neuron networks, and artificial neural networks, would likely have more elaborate representations, but we think it is instructive to consider this simple case first.}
We have two definitions, for feed-forward networks and networks with feedback.

%It seems that we don't need to do anything special here to accommodate overlap.

\subsubsection{Recognition in feed-forward networks}

The first definition assumes that $\mathcal N$ is a feed-forward network. 
In this definition, we specify not only which neurons fire, but also the precise time when they fire, which is just the time for firing to propagate to the neurons, step by step, through the network layers.

\begin{definition}
{\bf Recognition problem for feed-forward networks:}
\label{def: recog-ff}
Network $\mathcal N$ $(r_1,r_2)$-\emph{recognizes} $\mathcal C$
provided that, for each concept $c \in C$, there is a unique neuron
$rep(c)$ such that the following holds.
Assume that $B \subseteq C_0$ is presented at time $t$.  
Then:
\begin{enumerate}
\item
\emph{When $rep(c)$ must fire:}
If $c \in supp_{r_2}(B)$, then $rep(c)$ fires at time $t+layer(rep(c))$.
\item
\emph{When $rep(c)$ must not fire:}
If $c \notin supp_{r_1}(B)$, then $rep(c)$ does not fire at time $t + layer(rep(c))$.
\end{enumerate}
\end{definition}

The special case where $r_1 = r_2 = 1$ has a simple characterization:

\begin{lemma}
\label{lem: recog-special-case}
If network $\mathcal N$ $(1,1)$-recognizes $\mathcal C$, then for each concept $c \in C$, there is a unique neuron $rep(c)$ such that the following holds.
If $B \subseteq C_0$ is presented at time $t$, then $rep(c)$ fires at
time $t + layer(rep(c))$ if and only if $leaves(c) \subseteq B$.
\end{lemma}

\subsubsection{Recognition in networks with feedback}

The second definition assumes that $\mathcal N$ is a network with feedback.
For this, the timing is harder to pin down, so we formulate the definition a bit differently.
We assume here that the input is presented continually from some time $t$ onward, and we allow flexibility in when the $rep(c)$ neurons are required to fire.  

\begin{definition}
\label{def: recog-feedback}
{\bf Recognition problem for networks with feedback:}
Network $\mathcal N$ $(r_1,r_2,f)$-\emph{recognizes} $\mathcal C$
provided that, for each concept $c \in C$, there is a unique neuron
$rep(c)$ such that the following holds.
Assume that $B \subseteq C_0$ is presented at all times $\geq t$.
Then:
\begin{enumerate}
\item
\emph{When $rep(c)$ must fire:}
If $c \in supp_{r_2,f}(B)$, then $rep(c)$ fires at some time
after $t$. 
\item
\emph{When $rep(c)$ must not fire:}
If $c \notin supp_{r_1,f}(B)$, then $rep(c)$ does not fire at any
time after $t$.
\end{enumerate}
\end{definition}

% [[[Alternatively, for 1, we might say that it fires at all times from some time t' onward.  Would that be better?]]]

%%%%%%%%%%%%%%%%%%%%%%%%%%%%%%%%%%%%
\subsection{Learning Problems}
\label{sec:prob-learning}

In our learning problems, the same network ${\mathcal N}$ must be capable of learning any concept hierarchy $\mathcal C$.
The definitions are similar to those in~\cite{LM21}, but now we extend them to concept hierarchies with limited overlap. 

As before, we assume that the concepts are shown in a bottom-up manner, though interleaving is allowed for incomparable concepts.\footnote{We might also consider interleaved learning of higher-level concepts and their descendants.  The idea is that partial learning of a concept $c$ can be enough to make $rep(c)$ fire, which can help in learning parents of $c$. We mention this as future work, in Section~\ref{sec: conclusions}.}
% [[[When a concept has been partially learned, its rep responds to the firing of "most" of its children, but when it is more fully learned, the rep responds to fewer of the children.  It seems interesting to quantify this.]]]

\begin{definition}
{\bf Showing a concept:}
\label{def: shown}
Concept $c$  is \emph{shown} at time $t$ if the set $leaves(c)$ is presented at time $t$, that is, if for every input neuron $u$,
$u$ fires at time $t$ if and only if $u \in \{ rep(c') ~|~ c' \in leaves(c) \}$.
\end{definition}

\begin{definition}
{\bf Training schedule:}
\label{def: training-schedule}
A \emph{training schedule} for $\mathcal C$ is any finite sequence
$c_0,c_1,\ldots,c_m$ of concepts in $C$, possibly with repeats.
A training schedule is $\sigma$-\emph{bottom-up}, where $\sigma$ is a
positive integer, provided that the following conditions hold:
\begin{enumerate}
    \item Each concept in $C$ appears in the list at least $\sigma$ times.
    \item No concept in $C$ appears before each of its children has
      appeared at least $\sigma$ times.
      \end{enumerate}
\end{definition}

A training schedule $c_0, c_1,\ldots,c_m$ generates a
corresponding sequence $B_0,B_1,\ldots,B_m$ of sets of level $0$
concepts to be presented to the network in a learning algorithm. 
Namely, $B_i$ is defined to be $\{ rep(c') ~|~ c' \in leaves(c_i) \}$.

We have two definitions for learning, for networks with and without feedback.
The difference is just the type of recognition that is required to be achieved.  
Each definition makes sense with or without overlap.

\begin{definition}
{\bf Learning problem for feed-forward networks:}
Network $\mathcal N$ $(r_1,r_2)$-\emph{learns} concept hierarchy $\mathcal C$ \emph{with} $\sigma$ \emph{repeats} provided that the
following holds.
After a training phase in which all the concepts in $\mathcal C$ are shown to the network
according to a $\sigma$-bottom-up training schedule, $\mathcal N$
$(r_1,r_2)$-recognizes $\mathcal C$.
\end{definition}

% [[[A slight ambiguity here?  I mean right after the training schedule is finished.  I presume that that point the training stops, so that the network continues to recognize C.]]]

\begin{definition}
{\bf Learning problem for networks with feedback:}
Network $\mathcal N$ $(r_1,r_2,f)$-\emph{learns} concept hierarchy
$\mathcal C$ \emph{with} $\sigma$ \emph{repeats} provided that the
following holds.
After a training phase in which all the concepts in $\mathcal C$ are shown to the network
according to a $\sigma$-bottom-up training schedule, $\mathcal N$
$(r_1,r_2,f)$-recognizes $\mathcal C$.
\end{definition}

%%%%%%%%%%%%%%%%%%%%%%%%%%%%%%%%%%%%%%%%%%%%%%
\section{Recognition Algorithms for Feed-Forward Networks}
\label{sec: recognition-ff}

In this and the following section, we describe and analyze our algorithms for recognition of concept hierarchies (possibly with limited overlap);
we consider feed-forward networks in this section and introduce feedback edges in Section~\ref{sec: recognition-feedback}.
Throughout both sections, we consider an arbitrary concept hierarchy $\mathcal C$ with concept set $C$, partitioned into $C_0, C_1, \ldots, C_{\lmax}$.
We use the notation $n$, $k$, $r_1$, $r_2$, $o$, and $f$ as before.
We assume that $r_2 > 0$.

We begin in Section~\ref{sec: recognition-ff-exact} by defining a basic network, with weights in $\{0,1\}$, and proving that it $(r_1,r_2)$-recognizes $\mathcal C$.  
To prove this result, we use a new lemma that relates the firing behavior of the network precisely to the support definition, then obtain the main recognition theorem as a simple corollary.
In Section~\ref{sec: uncertain-weights1}, we extend the main result by allowing weights to be approximate, within an interval of uncertainty. 
In Section~\ref{sec: scaled}, we extend the results further by allowing scaling of weights and thresholds. 
In Appendix~\ref{sec: stochastic}, we discuss what happens in a different version of the model, where we replace thresholds by stochastic firing decisions.
%
%We close with Section~\ref{sec: stochastic}, which discusses what happens if we %allow the neurons to make probabilistic, rather than deterministic firing decisions.

\subsection{Basic Recognition Results}
\label{sec: recognition-ff-exact}

We define a feed-forward network $\mathcal N$ that is specially tailored to recognize concept hierarchy $\mathcal C$.
We assume that $\mathcal N$ has $\ell'_{max} = \lmax$ layers.
Since $\mathcal N$ is a feed-forward network, the edges all go upward, from neurons at any layer $\ell$ to neurons at level $\ell + 1$.
We assume the same value of $n$ as in the concept hierarchy $\mathcal C$.
The edge weights and the threshold $\tau$ are defined below.

The construction is similar to the corresponding construction in~\cite{LM21}.  The earlier paper considered only tree hierarchies; here, we generalize to allow limited overlap.

The strategy is simply to embed the digraph induced by $\mathcal C$ in the network $\mathcal N$.
For every level $\ell$ concept $c$ of $\cal{C}$, we assume a unique
representative $rep(c)$ in layer $\ell$ of the network.
Let $R$ be the set of all representatives, that is, $R = \{ rep(c)~|~c\in C\}$.
Let $rep^{-1}$ denote the corresponding inverse
function that gives, for every $u \in R$, the corresponding concept $c\in C$
with $rep(c) = u$.

We define the weights of the edges as follows.
If $u$ is any layer $\ell$ neuron, $0 \leq \ell \leq \lmax - 1$, and $v$ is any layer $\ell+1$ neuron, then we define the edge weight $weight(u,v)$ to be\footnote{Note that these simple weights of $1$ do not correspond precisely to what is achieved by the noise-free learning algorithm in~\cite{LM21}.
There, learning approaches the following weights, in the limit:
\[ weight(u,v) = 
\begin{cases}
\frac{1}{\sqrt{k}} & \text{ if } u,v \in R \text{ and } rep^{-1}(u) \in children(rep^{-1}(v)),
\\ 
0 & \text{ otherwise,}
\end{cases} \]
and the threshold $\tau$ is equal to $\frac{(r_1  + r_2) \sqrt{k}}{2} $.
We prove in~\cite{LM21} that, after a certain number of steps of learning, the weights are sufficiently close to these limits to guarantee that network $\mathcal N$ $(r_1,r_2)$-recognizes $\mathcal C$.}:
\[ weight(u,v) = 
\begin{cases}
1 & \text{ if } u,v \in R \text{ and } rep^{-1}(u) \in children(rep^{-1}(v)),
\\ 
0 & \text{ otherwise.}
\end{cases}
\]
% [[[I think that u and v were interchanged in our NN paper.  Should not cause confusion though, just a typo.]]]

We would like the threshold $\tau$ for every non-input neuron to be a real value in the closed interval $[r_1 k, r_2 k]$; to be specific, we use $\tau
= \frac{(r_1  + r_2) k}{2} $.
Since $r_2 > 0$, we know that $\tau > 0$.
Finally, we assume that the initial firing status for all non-input neurons is $0$.
This completely defines network $\mathcal N$, and determines its behavior.

The network has been designed in such a way that its behavior directly mirrors the $supp_r$ definition, where $r = \frac{\tau}{k}$.  We capture this precisely in the following two lemmas.  The first says that, when a subset of $C_0$ is presented, only $reps$ of concepts in $C$ fire at their designated times.

\begin{lemma}
\label{lem: no-rep-no-fire}
Assume $\mathcal C$ is any concept hierarchy satisfying the limited-overlap property, and $\mathcal N$ is the feed-forward network defined above, based on $\mathcal C$.  
Assume that $B \subseteq C_0$ is presented at time $t$. 
If $u$ is a neuron that fires at time $t+ layer(u)$, then $u \in R$, that is, $u = rep(c)$ for some concept $c \in C$.  
\end{lemma}

\begin{proof}
If $layer(u) = 0$, then $u$ fires at time $t$ exactly if $u = rep(c)$ for some $c \in B$, by assumption.
So consider $u$ with $layer(u) \geq 1$.
We show the contrapositive.
Assume that $u \notin R$.
Then $u$ has no positive weight incoming edges, by definition of the weights.
So $u$ cannot receive enough incoming potential for time $t+layer(u)$ to meet the positive firing threshold $\tau$.
\end{proof}

The second lemma says that the $rep$ of a concept $c$ fires at time $t + level(c)$ if and only if $c$ is supported by $B$.

\begin{lemma}
\label{lem: fires-supported-ff}
Assume $\mathcal C$ is any concept hierarchy satisfying the limited-overlap property, and $\mathcal N$ is the feed-forward network defined above, based on $\mathcal C$.  
Let $r = \frac{\tau}{k}$, where $\tau$ is the firing threshold for the non-input neurons of $\mathcal N$.

Assume that $B \subseteq C_0$ is presented at time $t$.  
If $c$ is any concept in $C$, then $rep(c)$ fires at time $t+level(c) (= t + layer(rep(c))$ if and only if $c \in supp_r(B)$.  
\end{lemma}

\begin{proof}
We prove the two directions separately.
\begin{enumerate}
\item If $c \in supp_{r}(B)$ then $rep(c)$ fires at time $t + level(c)$.

We prove this using induction on $level(c)$.
For the base case, $level(c) = 0$, the assumption that $c \in supp_r(B)$ means that $c \in B$, which means that $rep(c)$ fires at time $t$, by the assumption that $B$ is presented at time $t$.

For the inductive step, assume that $level(c) = \ell+1$.
Assume that $c \in supp_{r}(B)$.
Then by definition of $supp_r$, $c$ must have at least $r k$ children that are in $supp_{r}(B)$.  By inductive hypothesis, the $reps$ of all of these children fire at time $t + \ell$.
That means that the total incoming potential to $rep(c)$ for time $t+\ell+1$, $pot^{rep(c)}(t+\ell+1)$, reaches the firing threshold $\tau = r k$, so $rep(c)$ fires at time $t + \ell+1$.

\item 
If $c \notin supp_r(B)$ then $rep(c)$ does not fire at time $t+level(c)$.

Again, we use induction on $level(c)$.
For the base case, $level(c) = 0$, the assumption that $c \notin supp_r(B)$ means that $c \notin B$, which means that $rep(c)$ does not fire at time $t$ by the assumption that $B$ is presented at time $t$.

For the inductive step, assume that $level(c) = \ell+1$.
Assume that $c \notin supp_r(B)$.
Then $c$ has strictly fewer than $r k$ children that are in $supp_{r}(B)$,
and therefore, strictly more than $k - r k$ children that are not in 
$supp_{r}(B)$.
By inductive hypothesis, none of the $reps$ of the children in this latter set fire at time $t + \ell$, which means that the $reps$ of strictly fewer than $r k$ children of $c$ fire at time $t+\ell$.
So the total incoming potential to $rep(c)$ from $reps$ of $c$'s children is strictly less than $rk$.
Since only $reps$ of children of $c$ have positive-weight edges to $rep(c)$, that means that the total incoming potential to $rep(c)$ for time $t+\ell+1$, $pot^{rep(c)}(t+\ell+1)$, is strictly less than the threshold $\tau = r k$ for $rep(c)$ to fire at time $t + \ell+1$.
So $rep(c)$ does not fire at time $t+\ell+1$.
\end{enumerate}
\end{proof}

Using Lemma~\ref{lem: fires-supported-ff}, the basic recognition theorem follows easily:

\begin{theorem}
\label{th: recog-ff}
Assume $\mathcal C$ is any concept hierarchy satisfying the limited-overlap property, and $\mathcal N$ is the feed-forward network defined above, based on $\mathcal C$.  
Then ${\cal N}$ $(r_1,r_2)$-\emph{recognizes} $\cal{C}$.
\end{theorem}
Recall that the definition of recognition, Definition~\ref{def: recog-ff}, gives a firing requirement for each individual concept $c$ in the hierarchy.  For a concept $c$, the definition specifies that neuron $rep(c)$ fires at time $t+layer(rep(c)) = t + level(c)$, where $t$ is the time at which the input is presented. 

\begin{proof}
Let $r = \frac{\tau}{k}$, where $\tau$ is the firing threshold for the non-input neurons of $\mathcal N$.
Assume that $B \subseteq C_0$ is presented at time $t$.  We prove the two parts of Definition~\ref{def: recog-ff} separately.

\begin{enumerate}
\item If $c \in supp_{r_2}(B)$ then $rep(c)$ fires at time $t + level(c)$.

Suppose that $c \in supp_{r_2}(B)$.  
By assumption, $\tau \leq r_2 k$, so that $r = \frac{\tau}{k} \leq r_2$.
Then Lemma~\ref{lem: support-monotonic} implies that $c \in supp_{r}(B)$.
Then by Lemma~\ref{lem: fires-supported-ff}, $rep(c)$ fires at time $t + level(c)$.

\item If $c \notin supp_{r_1}(B)$ then $rep(c)$ does not fire at time $t + level(c)$.

Suppose that $c \notin supp_{r_1}(B)$.  
By assumption, $\tau \geq r_1 k$, so that $r = \frac{\tau}{k} \geq r_1$.
Then Lemma~\ref{lem: support-monotonic} implies that $c \notin supp_{r}(B)$.
Then by Lemma~\ref{lem: fires-supported-ff}, $rep(c)$ does not fire at time $t + level(c)$.
\end{enumerate}
\end{proof}

%%%%%%%%%%%%%%%%%%%%%%%%%%%%%%%%%%%%%%%%%
\subsection{An Issue Involving Overlap}

A new issue arises as a result of allowing overlap:
Consider two concepts $c$ and $c'$, with $level(c) = level(c')$.
Is it possible that showing concept $c'$ can cause $rep(c)$ to fire?
Specifically, suppose that concept $c'$ is shown at some time $t$, according to Definition~\ref{def: shown}.
That is, the set $leaves(c')$ is presented at time $t$.
Can this cause firing of $rep(c)$ at the designated time $t+level(c)$?
We obtain the following negative result.  For this, we assume that the amount of overlap is smaller than the lower bound for recognition.
% [[[Maybe I should put the restrictions in the general parameter list:  $0 \leq o < r_1 \leq r_2$.  That covers both the $o < r_1$ requirement and the $0 < r_2$ requirement I used earlier.]]]

\begin{theorem}
\label{th: overlap-ff}
Assume $\mathcal C$ is any hierarchy satisfying the limited-overlap property, and $\mathcal N$ is the feed-forward network defined above, based on $\mathcal C$.
Suppose that $o < r_1$.

Let $c,c'$ be two distinct concepts with $level(c') = level(c).$
Suppose that $c'$ is shown at time $t$.
Then $rep(c)$ does not fire at time $t + level(c)$.
\end{theorem}

\begin{proof}
Fix $c,c'$ as above, and assume that $c'$ is shown at time $t$.

\noindent
\emph{Claim:}
For any descendant $d$ of $c$ that is not also a descendant of $c'$, $rep(d)$ does not fire at time $t + level(d)$.

\noindent\emph{Proof of Claim:}
By induction on $level(d)$.
For the base case, $level(d) = 0$.  We know that $rep(d)$ does not fire at time $t$ because only $reps$ of descendants of $c'$ fire at time $t$.

For the inductive step, assume that $d$ is a descendant of $c$ that is not also a descendant of $c'$.
By the limited-overlap assumption, $d$ has at most $o \cdot k < r_1 k$ children that are also descendants of $c'$.
%[[[Need more details here?  All of $d$'s children that are also descendants of $c'$ must have other parents that are descendants of $c'$, and hence different from $d$.  Then the limited-overlap assumption applies. to limit the number of such children of $d$.]]]
%
By inductive hypothesis, the $reps$ of all the other children of $d$ do not fire at time $t + level(d) - 1$.   So the number of children of $d$ whose $reps$ fire at time $t+level(d)-1$ is strictly less than $r_1 k$.
That is not enough to meet the firing threshold $\tau \geq r_1 k$ for $rep(d)$ to fire at time $t + level(d)$. \\
\emph{End of proof of Claim.}

Applying the Claim with $d = c$ yields that $rep(c)$ does not fire at time $t + level(c)$.
\end{proof}

{Note that Theorem~\ref{th: overlap-ff} can be generalized to the situation where any set of concepts at $level(c)$, not including $c$ itself, are shown simultaneously.}

We do not have a result analogous to Theorem~\ref{th: overlap-ff} for recognition in networks with feedback (as in Section~\ref{sec: recognition-feedback}).  The situation in that case is more complicated, and we leave that for future work.

\subsection{Approximate Weights} 
\label{sec: uncertain-weights1}

So far in this section, we have been considering a simple set of weights, for a network representing a particular concept hierarchy $\mathcal C$:
\[ weight(u,v) = 
\begin{cases}
1 & \text{ if } u,v \in R \text{ and } rep^{-1}(u) \in children(rep^{-1}(v)),
\\ 
0 & \text{ otherwise.}
\end{cases}
\]

Now we generalize by allowing the weights to be specified only approximately, within some interval.
This is useful, for example, when the weights result from a noisy learning process.
Here, we assume $0 \leq w_1 \leq w_2$, and allow $b$ to be any positive integer.

\[ weight(u,v) \in 
\begin{cases}
[w_1,w_2], & \text{ if } u,v \in R \text{ and } rep^{-1}(u) \in children(rep^{-1}(v)),
\\ 
[0,\frac{1}{k^{\lmax+b}}]  & \text{ otherwise.}
\end{cases}
\]
Again, we set threshold $\tau = \frac{(r_1 + r_2) k}{2}$.  And we add the (extremely trivial) assumption that $\tau > 1/k^{b-1}$.
For this case, we prove the following recognition result.  It relies on two inequalities involving the recognition bounds and the weight bounds.

\begin{theorem}
\label{th: overlap-ff-uncertain weights}
Assume $\mathcal C$ is any concept hierarchy satisfying the limited-overlap property, and $\mathcal N$ is the feed-forward network defined above, based on $\mathcal C$.
Assume that $\frac{(r_1 + r_2)k}{2} > \frac{1}{k^{b-1}}$.
Suppose that $r_1$ and $r_2$ satisfy the following inequalities:
\begin{enumerate}
    \item $r_2 (2 w_1 - 1) \geq r_1$.
    \item $r_2 \geq r_1(2w_2 - 1) + \frac{2}{k^b}$.
\end{enumerate}
Then ${\cal N}$ $(r_1,r_2)$-\emph{recognizes} $\mathcal C$.
\end{theorem}
% [[[First inequality comes from:  $r_2 k w_1 \geq \tau = \frac{(r_1 + r_2) k}{2}$.
% That is, $r_2 w_1 geq \frac{(r_1 + r_2)}{2}.
% Second inequality comes from:  $r_1 k w_2 + \frac{k^lmax+1} k^lmax+1} \leq tau = \frac{(r_1 + r_2) k}{2}$$.
% That is, $r_1 w_2 + 1/k \leq \frac{(r_1 + r_2)}{2}$$.
%]]]

To prove Theorem~\ref{th: overlap-ff-uncertain weights}, we follow the general pattern of the proof of Theorem~\ref{th: recog-ff}.  We use versions of Lemma~\ref{lem: no-rep-no-fire} and~\ref{lem: fires-supported-ff}:

\begin{lemma}
\label{lem: no-rep-no-fire-uncertain}
Assume $\mathcal C$ is any concept hierarchy satisfying the limited-overlap property, and $\mathcal N$ is the feed-forward network defined above, based on $\mathcal C$.  
Assume that $\frac{(r_1 + r_2)k}{2} > \frac{1}{k^{b-1}}$.

Assume that $B \subseteq C_0$ is presented at time $t$. 
If $u$ is a neuron that fires at time $t+ layer(u)$, then $u = rep(c)$ for some concept $c \in C$.  
\end{lemma}

\begin{proof}
The proof is slightly more involved than the one for Lemma~\ref{lem: no-rep-no-fire}.
This time we proceed by induction on $layer(u)$.
For the base case, 
If $layer(u) = 0$, then $u$ fires at time $t$ exactly if $u = rep(c)$ for some $c \in B$, by assumption.

For the inductive step, consider $u$ with $layer(u) = \ell+1$.
Assume for contradiction that $u$ is not of the form $rep(c)$ for any $c \in C$.
Then the weight of each edge incoming to $u$ is at most $\frac{1}{k^{\lmax+b}}$.
By inductive hypothesis, the only layer $\ell$ incoming neighbors that fire at time $t+\ell$ are $reps$ of concepts in $C$.
There are at most $k^{\lmax+1}$ such concepts, hence at most $k^{\lmax+1}$ level $\ell$ incoming neighbors fire at time $t+\ell$, yielding a total incoming potential for $u$ for time $t+\ell+1$ of at most $\frac{k^{\lmax+1}}{k^{\lmax+b}} = \frac{1}{k^{b-1}}$.
Since the firing threshold $\tau = \frac{(r_1 + r_2)k}{2}$ is strictly greater than  $\frac{1}{k^{b-1}}$, $u$ cannot receive enough incoming potential to meet the threshold for time $t+\ell+1$.
\end{proof}

\begin{lemma}
\label{lem: fires-supported-ff-uncertain}
Assume $\mathcal C$ is any concept hierarchy satisfying the limited-overlap property and $\mathcal N$ is the feed-forward network defined above, based on $\mathcal C$.
Assume that $\frac{(r_1 + r_2)k}{2} > \frac{1}{k^{b-1}}$.
Suppose that $r_1$ and $r_2$ satisfy the following inequalities:
\begin{enumerate}
    \item $r_2 (2 w_1 - 1) \geq r_1$.
    \item $r_2 \geq r_1(2w_2 - 1) + \frac{2}{k^b}$.
\end{enumerate}
Assume that $B \subseteq C_0$ is presented at time $t$.
If $c$ is any concept in $C$, then
\begin{enumerate}
    \item If $c \in supp_{r_2}(B)$ then $rep(c)$ fires at time $t+level(c)$.
    \item If $c \notin supp_{r_1}(B)$ then $rep(c)$ does not fire at time $t+level(c)$.
\end{enumerate}
\end{lemma}

\begin{proof}
\begin{enumerate}
\item If $c \in supp_{r_2}(B)$ then $rep(c)$ fires at time $t + level(c)$.

We prove this using induction on $level(c)$.
For the base case, $level(c) = 0$, the assumption that $c \in supp_{r_2}(B)$ means that $c \in B$, which means that $rep(c)$ fires at time $t$, by the assumption that $B$ is presented at time $t$.

For the inductive step, assume that $level(c) = \ell+1$.
Assume that $c \in supp_{r_2}(B)$.
Then $c$ must have at least $r_2 k$ children that are in $supp_{r_2}(B)$.  
By inductive hypothesis, the $reps$ of all of these children fire at time $t + \ell$.

We claim that the total incoming potential to $rep(c)$ for time $t+\ell+1$, $pot^{rep(c)}(t+\ell+1)$, reaches the firing threshold $\tau = \frac{(r_1+r_2)k}{2}$, so $rep(c)$ fires at time $t+\ell+1$.
To see this, note that the total potential induced by the firing $reps$ of children of $c$ is at least $r_2 k w_1$, because the weight of the edge from each firing child to $rep(c)$ is at least $w_1$.
This quantity is $\geq \frac{(r_1+r_2)k}{2}$ because of the assumption that
$r_2 (2 w_1 - 1) \geq r_1$.
%
%[[[Since $r_2 (2 w_1 - 1) \geq r_1$, we have $r_2 (w_1 - 1/2) \geq r_1/2$.
%Which implies that $r_2 w_1 \geq r_1+r_2/2$.

\item If $c \notin supp_{r_1}(B)$ then $rep(c)$ does not fire at time $t + level(c)$.

Again we use induction on $level(c)$.
For the base case, $level(c) = 0$, the assumption that $c \notin supp_{r_1}(B)$ means that $c \notin B$, which means that $rep(c)$ does not fire at time $t$, by the assumption that $B$ is presented at time $t$.

For the inductive step, assume that $level(c) = \ell+1$.  
Assume that $c \notin supp_{r_1}(B)$.
Then $c$ has strictly fewer than $r_1 k$ children that are in $supp_{r_1}(B)$,
and therefore, strictly more than $k - r_1 k$ children that are not in 
$supp_{r_1}(B)$.
By inductive hypothesis, none of the $reps$ of the children in this latter set fire at time $t + \ell$, which means that the $reps$ of strictly fewer than $r_1 k$ children of $c$ fire at time $t + \ell$.
Therefore, the total incoming potential to $rep(c)$ from $reps$ of $c$'s children is strictly less than $r_1 k w_2$, since the weight of the edge from each firing child to $rep(c)$ is at most $w_2$.

In addition, some potential may be contributed by other neurons at level $\ell$ that are not children of $c$ but fire at time $t+\ell-1$.
By Lemma~\ref{lem: no-rep-no-fire-uncertain}, these must all be $reps$ of concepts in $C$.
There are at most $k^{\lmax+1}$ of these, each contributing potential of at most $\frac{1}{k^{lmax+b}}$, for a total potential of at most $\frac{1}{k^{b-1}}$ from these neurons.

Therefore, the total incoming potential to $rep(c)$ for time $t+\ell+1$, $pot^{rep(c)}(t+\ell+1)$, is strictly less than $r_1 k w_2 + \frac{1}{k^{b-1}}$.
This quantity is $\leq \frac{(r_1+r_2)k}{2}$, 
because of the assumption that $r_2 \geq r_1(2w_2 - 1) + \frac{2}{k^b}$.
This means that the total incoming potential to $rep(c)$ for time $t+\ell+1$ is strictly less than the threshold $\tau = \frac{(r_1+r_2)k}{2}$ for $rep(c)$ to fire at time $t+\ell+1$.
So $rep(c)$ does not fire at time $t+\ell+1$.
\end{enumerate}
\end{proof}

Now we can prove Theorem~\ref{th: overlap-ff-uncertain weights}.

\begin{proof}
(Of Theorem~\ref{th: overlap-ff-uncertain weights}:)
The proof is similar to that of Theorem 5.3, but now using Lemma~\ref{lem: fires-supported-ff-uncertain} in place of Lemma~\ref{lem: fires-supported-ff}.
Let $r = \frac{\tau}{k}$, where $\tau$ is the firing threshold for the non-input neurons of $\mathcal N$.
Assume that $B \subseteq C_0$ is presented at time $t$.  We prove the two parts of Definition~\ref{def: recog-ff} separately.

\begin{enumerate}
\item If $c \in supp_{r_2}(B)$ then $rep(c)$ fires at time $t + level(c)$.

Suppose that $c \in supp_{r_2}(B)$.  
By assumption, $\tau \leq r_2 k$, so that $r = \frac{\tau}{k} \leq r_2$.
Then Lemma~\ref{lem: support-monotonic} implies that $c \in supp_{r}(B)$.
Then Lemma~\ref{lem: fires-supported-ff-uncertain} implies that $rep(c)$ fires at time $t + level(c)$.

\item If $c \notin supp_{r_1}(B)$ then $rep(c)$ does not fire at time $t + level(c)$.

Suppose that $c \notin supp_{r_1}(B)$.  
By assumption, $\tau \geq r_1 k$, so that $r = \frac{\tau}{k} \geq r_1$.
Then Lemma~\ref{lem: support-monotonic} implies that $c \notin supp_{r}(B)$.
Then Lemma~\ref{lem: fires-supported-ff-uncertain} implies that $rep(c)$ does not fire at time $t + level(c)$.
\end{enumerate}
\end{proof}

%  [[[Suppose that $w_1 \leq 1 \leq w_2$.  As we learn over time, we expect that the bounds $w_1$ and $w_2$ would converge to $1$.  This should allow for less restrictive constraints on $r_1$ and $r_2$, e.g. $r_2 \geq r_1 + 2/k$.  Actually, the $2/k$ can be made smaller if the upper bound for the non-child edges is required to be much less than $\frac{1}{k^{lmax+1}}$, say $\frac{1}{k^{lmax+b}}$ for some sufficiently larger $b$.  Such a $b$ appears in LM21.  TBD.]]]

%%%%%%%%%%%%%%%%%%%%%%%%%%%%%%%%%%%%%%%%
\subsection{Scaled Weights and Thresholds} 
\label{sec: scaled}

Our recognition results in Section~\ref{sec: uncertain-weights1} assume a firing threshold of $\frac{(r_1+r_2)k}{2}$ and bounds $w_1$ and $w_2$ on weights on edges from children to parents.
The form of the two inequalities in Theorem~\ref{th: overlap-ff-uncertain weights} suggests that $w_1$ and $w_2$ should be close to $1$, because in that case the two parenthetical expressions are close to $1$ and the constraints on the values of $r_1$ and $r_2$ are weak.

On the other hand, the noise-free learning results in~\cite{LM21} assume a threshold of $\frac{(r_1+r_2) \sqrt{k}}{2}$, that is, our threshold in Section~\ref{sec: uncertain-weights1} is multiplied by $\frac{1}{\sqrt{k}}$.  
Also, in~\cite{LM21}, the weights on edges from children to parents approach $\frac{1}{\sqrt{k}}$ in the limit rather than $1$, because of the behavior induced by Oja's rule.

%We expect that the results from Section 5 will carry over to the setting of~\cite{LM21}, with some small adjustments.  Specifically, assume that the threshold is $\frac{(r_1+r_2) \sqrt{k}}{2}$, and the learned weights for edges from children to parents are $\frac{1}{\sqrt{k}}$ and $0$, rather than $1$ and $0$.
%That is, the threshold and the learned weights for children get divided by $\sqrt{k}$.
%Then Theorem~\ref{th: recog-ff} carries over with only minor scaling changes.

We would like to view the results of noise-free learning in terms of achieving a collection of weights that suffice for recognition.  That is, we would like a version of Theorem~\ref{th: overlap-ff-uncertain weights} for the case where the
threshold is $\frac{(r_1+r_2) \sqrt{k}}{2}$ and the weights are:

\[ weight(u,v) \in 
\begin{cases}
[\frac{w_1}{\sqrt{k}}, \frac{w_2}{\sqrt{k}}], & \text{ if } u,v \in R \text{ and } rep^{-1}(u) \in children(rep^{-1}(v)),
\\ 
[0,\frac{1}{k^{\lmax+b}}]  & \text{ otherwise.}
\end{cases}
\]

Here, we assume that $w_1 \leq w_2$ and both are close to $1$.
%Here we need to adjust the second equality relating the parameters.
%to $r_2 \geq r_1(2 w_2 - 1) + \frac{2}{\sqrt{k}}$.

More generally, we can scale by multiplying the threshold and weights by a constant \emph{scaling factor} $s$, $0 < s < 1$, in place of $\frac{1}{\sqrt{k}}$, giving a threshold of $\frac{(r_1+r_2) k s}{2}$ and weights of:
\[ weight(u,v) \in 
\begin{cases}
[w_1 s, w_2 s], & \text{ if } u,v \in R \text{ and } rep^{-1}(u) \in children(rep^{-1}(v)),
\\ 
[0,\frac{1}{k^{\lmax+b}}]  & \text{ otherwise.}
\end{cases}
\]

For this general case, we get a new version of Theorem~\ref{th: overlap-ff-uncertain weights}:

%Then we get new versions of Theorem~\ref{th: recog-ff} and~\ref{th: overlap-ff-uncertain weights}:

%\begin{theorem}
%\label{th: recog-ff-scaled}
%Assume $\mathcal C$ is any concept hierarchy satisfying limited overlap, and $\mathcal N$ is the %feed-forward network defined above, based on $\mathcal C$.  
%Then ${\cal N}$ $(r_1,r_2)$-\emph{recognizes} $\cal{C}$.
%\end{theorem}

\begin{theorem}
\label{th: overlap-ff-uncertain weights-scaled}
Assume $\mathcal C$ is any concept hierarchy satisfying the limited-overlap property, and $\mathcal N$ is the feed-forward network defined above (with weights scaled by an arbitrary $s$).
Assume that $\frac{(r_1 + r_2)ks}{2} > \frac{1}{k^{b-1}}$.

Suppose that $r_1$ and $r_2$ satisfy the following inequalities:
\begin{enumerate}
    \item $r_2 (2 w_1 - 1) \geq r_1$.
    \item $r_2 \geq r_1(2 w_2 - 1) + \frac{2}{k^b s}$.
\end{enumerate}
Then ${\cal N}$ $(r_1,r_2)$-\emph{recognizes} $\mathcal C$.
\end{theorem}

Theorem~\ref{th: overlap-ff-uncertain weights-scaled} can be proved by slightly adjusting the proofs of Lemma~\ref{lem: fires-supported-ff-uncertain} and Theorem~\ref{th: overlap-ff-uncertain weights}.

Using Theorem~\ref{th: overlap-ff-uncertain weights-scaled}, we can decompose the proof of correctness for noise-free learning in~\cite{LM21}, first showing that the learning algorithm achieves the weight bounds specified above, and then invoking the theorem to show that correct recognition is achieved.
For this, we use weights $w_1 = \frac{1}{1 + \epsilon}$ and $w_2 = 1$ and scaling factor $s = \frac{1}{\sqrt{k}}$.  The value of $\epsilon$ is an arbitrary element of $(0,1]$; the running time of the algorithm depends on $\epsilon$.

%%%%%%%%%%%%%%%%%%%%%%%%%%%%%%%%%%%%%%%%%%%%%%%%%%%%%%%%%%%%%%%%%%%%%%%%%%%%%%%%%%%%%%%%%%%
\section{Recognition Algorithms for Networks with Feedback}
\label{sec: recognition-feedback}

In this section, we assume that our network $\mathcal N$ includes downward edges, from every neuron in any layer to every neuron in the layer below.
We begin in Section~\ref{sec: recog-feedback-basic} with recognition results for a basic network, with upward weights in $\{0,1\}$ and downward weights in $\{0,f\}$.
Again, we prove these using a new lemma that relates the firing behavior of the network precisely to the support definition.

Recall that for networks with feedback, unlike feed-forward networks, the recognition definition does not specify precise firing times for the $rep$ neurons.
Therefore, in Sections~\ref{sec: time-bounds-trees-feedback} and~\ref{sec: time-bounds-recognition-overlap-feedback},
we prove time bounds for recognition; these bounds are different for tree hierarchies vs. general  hierarchies. 
Finally, in Section~\ref{sec: uncertain-weights-2}, we extend the main recognition result by allowing weights to be approximate, within an interval of uncertainty.
Extension to scaled weights and thresholds should also work in this case.

%%%%%%%%%%%%%%%%%%%%%%%%%%%%%%%%%%%%%%%%
\subsection{Basic Recognition Results}
\label{sec: recog-feedback-basic}

As before, we define a network $\mathcal N$ that is specially tailored to recognize concept hierarchy $\mathcal C$.
We assume that $\mathcal N$ has $\ell'_{max} = \lmax$.
We assume the same values of $n$ and $f$ as in $\mathcal C$.
As before, concept hierarchy $\mathcal C$ is embedded, one level per layer, in the network $\mathcal N$. 

Now we define edge weights for both upward and downward edges.
Let $u$ be any layer $\ell$ neuron and $v$ any layer $\ell+1$ neuron.
We define the weight for the upward edge $(u,v)$ as before:
\[ weight(u,v) = 
\begin{cases}
1 & \text{ if } u,v \in R \text{ and } rep^{-1}(u) \in children(rep^{-1}(v)),
\\ 
0 & \text{ otherwise.}
\end{cases} \]
For the downward edge $(v,u)$, we define:
\[ weight(v,u) = 
\begin{cases}
f & \text{ if } u,v \in R \text{ and } rep^{-1}(u) \in children(rep^{-1}(v)),
\\ 
0 & \text{ otherwise.}
\end{cases} \]
Thus, the weight of $f$ on the downward edges corresponds to the weighting factor of $f$ in the $supp_{r,f}$ definition.
As before, we set the threshold $\tau$ for every non-input neuron to be a real value in the closed interval $[r_1 k, r_2 k]$, specifically,  $\tau = \frac{(r_1  + r_2) k}{2} $.
Again, we assume that the initial firing status of all non-input neurons is $0$.

As before, we have:

\begin{lemma}
\label{lem: no-rep-no-fire-2}
Assume $\mathcal C$ is any concept hierarchy satisfying the limited-overlap property, and $\mathcal N$ is the network defined above, based on $\mathcal C$.  
Assume that $B \subseteq C_0$ is presented at time $t$. 
If $u$ is a neuron that fires at some time after $t$, then $u \in R$, that is, $u = rep(c)$ for some concept $c \in C$.  
\end{lemma}

The following preliminary lemma says that the firing of $rep$ neurons is persistent,  assuming persistent inputs (as we do in the definition of recognition for networks with feedback).

\begin{lemma}
\label{lem: persistent firing}
Assume $\mathcal C$ is any concept hierarchy satisfying the limited-overlap property, and $\mathcal N$ is the network with feedback defined above, based on $\mathcal C$. 
Let $r = \frac{\tau}{k}$, where $\tau$ is the firing threshold for the non-input neurons of $\mathcal N$.
  
Assume that $B \subseteq C_0$ is presented at all times $\geq t$.
Let $c$ be any concept in $C$.
Then for every $t' \geq t$, if $rep(c)$ fires at time $t'$, then it fires at all times $\geq t'$.
\end{lemma}

\begin{proof}
We prove this by induction on $t'$.
The base case is $t' = t$.
The neurons that fire at time $t$ are exactly the input neurons that are $reps$ for concepts in $B$.
By assumption, these same inputs continue for all times $\geq t$.

For the inductive step, consider a neuron $rep(c)$ that fires at time $t'$, where $t' \geq t+1$.
If $level(c) = 0$ then $c \in B$ and $rep(c)$ continues firing forever.
So assume that $level(c) \geq 1$.
Then $rep(c)$ fires at time $t'$ because the incoming potential it receives from its children and parents who fire at time $t'-1$ is sufficient to reach the firing threshold $\tau$.
By inductive hypothesis, all of the neighbors of $rep(c)$ that fire at time $t'-1$ also fire at all times $\geq t'-1$.
So that means that they provide enough incoming potential to $rep(c)$ to make $rep(c)$ fire at all times $\geq t'$. 
\end{proof}

Next we have a lemma that is analogous to Lemma~\ref{lem: fires-supported-ff}, but now in terms of eventual firing rather than firing at a specific time.
Similarly to before, this works because the network's behavior directly mirrors the $supp_{r,f}$ definition, where $r = \frac{\tau}{k}$.
%The weights of $f$ on the downward edges correspond to the weighting factor of $f$ in the $supp_{r,f}$ definition.

\begin{lemma}
\label{lem: fires-supported-feedback} 
Assume $\mathcal C$ is any concept hierarchy satisfying the limited-overlap property, and $\mathcal N$ is the network with feedback defined above, based on $\mathcal C$.  
Let $r = \frac{\tau}{k}$, where $\tau$ is the firing threshold for the non-input neurons of $\mathcal N$.

Assume that $B \subseteq C_0$ is presented at all times $\geq t$.
If $c$ is any concept in $C$, then $rep(c)$ fires at some time $\geq t$ if and only if $c \in supp_{r,f}(B)$.
\end{lemma}

To prove Lemma~\ref{lem: fires-supported-feedback}, it is convenient to prove a more precise version that takes time into account.
As before, in Section~\ref{sec: time-bounds}, we use the abbreviation $S(t) = supp_{r,f}(B,*,t)$.
Thus, $S(t)$ is the set of concepts at all levels that are supported by input $B$ by step $t$ of the recursive definition of the $S(\ell,t)$ sets.

\begin{lemma}
\label{lem: fires-supported-feedback-1} 
Assume $\mathcal C$ is any concept hierarchy satisfying the limited-overlap property, and $\mathcal N$ is the network with feedback defined above, based on $\mathcal C$.  
Let $r = \frac{\tau}{k}$, where $\tau$ is the firing threshold for the non-input neurons of $\mathcal N$.

Assume that $B \subseteq C_0$ is presented at all times $\geq t$.
Let $t'$ be any time $\geq t$.
If $c$ is any concept in $C$, then $rep(c)$ fires at time $t'$ if and only if $c \in S(t'-t)$.
\end{lemma}

\begin{proof}
As for Lemma~\ref{lem: fires-supported-ff}, we prove the two directions separately.  But now we use induction on time rather than on $level(c)$.

\begin{enumerate}
    \item  If $c \in S(t'-t)$, then $rep(c)$ fires at time $t'$.

    We prove this using induction on $t'$, $t' \geq t$.
    For the base case, $t' = t$, the assumption that $c \in S(0)$ means that $c$ is in the input set $B$, which means that $rep(c)$ fires at time $t$.
    
    For the inductive step, assume that $t' \geq t$ and $c \in S((t'+1)-t)$.  If $level(c) = 0$ then again $c \in B$, so $c$ fires at time $t$, and therefore at time $t'$ by Lemma~\ref{lem: persistent firing}.
    So assume that $level(c) \geq 1$.
    If $c \in S(t'-t)$ then $rep(c)$ fires at time $t'$ by the inductive hypothesis, and therefore also at time $t'+1$ by Lemma~\ref{lem: persistent firing}.
    Otherwise, enough of $c$'s children and parents must be in $S(t'-t)$ to include $c$ in $S((t'-t)+1) = S((t'+1)-t)$; that is, 
    $|children(c) \ \cap \ S(t'-t)| + f \ |parents(c)\  \cap \ S(t'-t)| \geq rk$.
    
    Then by inductive hypothesis, all of the $reps$ of the children and parent concepts mentioned in this expression fire at time $t'$.
    Therefore, the upward potential incoming to $rep(c)$ for time $t'+1$, $upot^{rep(c)}(t'+1)$, is at least $|children(c) \ \cap \ S(t'-t)|$, and the downward potential incoming to $rep(c)$ for time $t'+1$, $dpot^{rep(c)}(t'+1)$, is at least $f \ |parents(c) \ \cap \ S(t'-t)|$ (since the weight of each downward edge is $f$).
    So $pot^{rep(c)}(t'+1)$, which is equal to $upot^{rep(c)}(t'+1) + dpot^{rep(c)}(t'+1)$, is $\geq |children(c) \ \cap\ S(t'-t)| + f \  |parents(c) \ \cap \ S(t'-t)| \geq rk$.
    That reaches the firing threshold $\tau = rk$ for $rep(c)$ to fire at time $t'+1$.
    
    \item If $rep(c)$ fires at time $t'$, then $c \in S(t'-t)$.

    We again use induction on $t'$, $t' \geq t$. 
    For the base case, $t' = t$, the assumption that  $rep(c)$ fires at time $t$ means that $c$ is in the input set $B$, hence $c \in S(0)$.
    
    For the inductive step, suppose that $t' \geq t$ and $rep(c)$ fires at time $t' + 1$.
    Then it must be that enough of the $reps$ of $c$'s children and parents fire at time $t'$ to reach the firing threshold $\tau = rk$ for $rep(c)$ to fire at time $t'+1$.
    That is, $upot^{rep(c)}(t'+1) + dpot^{rep(c)}(t'+1) \geq rk$.
    In other words, the  number of $reps$ of children of $c$ that fire at time $t'$ plus $f$ times the number of $reps$ of parents of $c$ that fire at time $t'$ is $\geq r k$ (since the weight of each downward edge is $f$).
    
    By inductive hypothesis, all of these children and parents of $c$ are in $S(t'-t)$.  
    Therefore, $|children(c) \ \cap \ S(t'-t)| + f \ |parents(c) \ \cap \ S(t'-t)| \geq rk$.
    Then by definition of $supp_{r,f}(B)$, we have that $c \in S((t'-t)+1) = S((t' + 1) - t)$, as needed.
\end{enumerate}
\end{proof}

Lemma~\ref{lem: fires-supported-feedback} follows immediately from Lemma~\ref{lem: fires-supported-feedback-1}.
Then, as in Section~\ref{sec: recog-feedback-basic}, the main recognition theorem follows easily.

\begin{theorem}
Assume $\mathcal C$ is any concept hierarchy satisfying the limited-overlap property, and $\mathcal N$ is the network with feedback defined above, based on $\mathcal C$.
Then $\mathcal N$ $(r_1,r_2,f)-$ recognizes $\mathcal C$.
\end{theorem}

\begin{proof}
Let $r = \frac{\tau}{k}$, where $\tau$ is the firing threshold for the non-input neurons of $\mathcal N$.
Assume that $B \subseteq C_0$ is presented at all times $\geq t$.  We prove the two parts of Definition~\ref{def: recog-feedback} separately.
\begin{enumerate}
\item
If $c \in supp_{r_2,f}(B)$ then $rep(c)$ fires at some time $\geq t$. 

Suppose that $c \in supp_{r_2,f}(B)$.  
By assumption, $\tau \leq r_2k$, so that $r = \frac{\tau}{k} \leq r_2$.
Then Lemma~\ref{lem: support-monotonic-2} implies that $c \in supp_{r,f}(B)$.
Then by Lemma~\ref{lem: fires-supported-feedback}, $rep(c)$ fires at some time $\geq t$.

\item
If $c \notin supp_{r_1,f}(B)$ then $rep(c)$ does not fire at any time $\geq t$.

Suppose that $c \notin supp_{r_1,f}(B)$.  
By assumption, $\tau \geq r_1 k$, so that $r = \frac{\tau}{k} \geq r_1$.
Then Lemma~\ref{lem: support-monotonic-2} implies that $c \notin supp_{r,f}(B)$.
Then by Lemma~\ref{lem: fires-supported-feedback}, $rep(c)$ does not fire at any time $\geq t$.
\end{enumerate}
\end{proof}

%%%%%%%%%%%%%%%%%%%%%%%%%%%%%%%%%%%%%%%%
\subsection{Time Bounds for Tree Hierarchies in Networks with Feedback}
\label{sec: time-bounds-trees-feedback}

It remains to prove time bounds for recognition for hierarchical concepts in networks with feedback.  Now the situation turns out to be quite different for tree hierarchies and hierarchies that allow limited overlap.  In this section, we consider the simpler case of tree hierarchies.

For a tree network, one pass upward and one pass downward is enough to recognize all concepts, though that is a simplification of what actually happens, since much of the recognition activity is concurrent.  Still, for tree hierarchies, we can prove an upper bound of twice the number of levels:

\begin{theorem}
\label{th: time-bound-trees-feedback-1}
Assume $\mathcal C$ is a tree hierarchy and $\mathcal N$ is the network with feedback defined above, based on $\mathcal C$.  
Let $r = \frac{\tau}{k}$, where $\tau$ is the firing threshold for the non-input neurons of $\mathcal N$.

Assume that $B \subseteq C_0$ is presented at all times $\geq t$.
If $c \in supp_{r,f}(B)$, then $rep(c)$ fires at some time $\leq t + 2 \lmax$.
\end{theorem}

\begin{proof}
Assume that $c \in supp_{r,f}(B)$.  
By Lemma~\ref{lem: support-firing}, we have that $c \in S(2 \lmax)$.
Then Lemma~\ref{lem: fires-supported-feedback-1} implies that $rep(c)$ fires at time $t + 2 \lmax$.
\end{proof}

And this result extends to larger thresholds:

\begin{corollary}
\label{cor: time-bound-trees-feedback}
Assume $\mathcal C$ is a tree hierarchy and $\mathcal N$ is the network with feedback as defined above, based on $\mathcal C$.
Assume that $B \subseteq C_0$ is presented at all times $\geq t$.
If $c \in supp_{r_2,f}(B)$, then $rep(c)$ fires at some time $\leq t + 2 \lmax$.
\end{corollary}

\begin{proof}
By Theorem~\ref{th: time-bound-trees-feedback-1} and Lemma~\ref{lem: support-monotonic-2}.
\end{proof}

\subsection{Time Bounds for General Hierarchies in Networks with Feedback}
\label{sec: time-bounds-recognition-overlap-feedback}

The situation gets more interesting when the hierarchy allows overlap.
We use the same network as before.
Each neuron gets inputs from its children and
parents at each round, and fires whenever its threshold is met.
As noted in Lemma~\ref{lem: fires-supported-feedback-1}, this network behavior follows the definition of  $supp_{r_2,f}(B)$.

In the case of a tree hierarchy, one pass upward followed by one pass downward suffice to recognize all concepts, though the actual execution involves more concurrency, rather than separate passes.
But with overlap, more complicated behavior can occur.
For example, an initial pass upward can activate some $rep$ neurons, which can then provide feedback on a downward pass to activate some other $rep$ neurons that were not activated in the upward pass.   
So far, this is as for tree hierarchies.
But now because of overlap, these newly-recognized concepts can in turn trigger more recognition on another upward pass, then still more on another downward pass, etc.
How long does it take before the network is guaranteed to stabilize?  

Here we give a simple upper bound and an example that yields a lower bound. Work is needed to pin the bound down more precisely.

%For example:
%Suppose that each of $r_2 k$ of the (level $1$) children of $c$ has at
%least $r_2 k$ of its children in $B$.
%Then $B$ provides enough support for $c$ to fire.
%Now consider another child $c'$, one that has only $(r_2 - f)k$
%children in $B$.
%This need not fire initially based on the presentation of its own
%children.
%However, the firing of $c$ will contribute $\frac{f}{\sqrt{k}}$
%potential on the downward edge from $c$ to $c'$.
%That is enough to cause $c'$ to fire also.

\subsubsection{Upper bound}

We give a crude upper bound on the time to recognize all the concepts in a hierarchy.

\begin{theorem}
\label{th: time-bound-overlap-feedback}
Assume $\mathcal C$ is any hierarchy satisfying the limited-overlap property, and $\mathcal N$ is the network with feedback defined above, based on $\mathcal C$.
Assume that $B \subseteq C_0$ is presented at all times $\geq t$.
If $c \in supp_{r,f}(B)$, then $rep(c)$ fires at some time $\leq t + k^{\lmax+1}$.
\end{theorem}
% [[[Could have a corollary for $r_2$, probably not worth it.]]]

\begin{proof}
All the level $0$ concepts in $B$ start firing at time $0$.
We consider how long it might take, in the worst case, for the $reps$ of all the concepts in $supp_{r,f}(B)$ with levels $\geq 1$ to start firing.

The total number of concepts in $C$ with levels $\geq 1$ is at most $k^{\lmax+1}$; therefore, the number of concepts in $supp_{r_,f}(B)$ with levels $\geq 1$ is at most $k^{\lmax+1}$.

By Lemma~\ref{lem: no-rep-no-fire-2}, the $rep$ neurons are the only ones that ever fire.
Therefore, the firing set stabilizes at the first time $t'$ such that the sets of $rep$ neurons that fire at times $t'$ and time $t'+1$ are the same.
Since there are at most $k^{\lmax+1}$ $rep$ neurons with levels $\geq 1$, the worst case is if one new $rep$ starts firing at each time.  
But in this case the firing set stabilizes by $t + k^{\lmax+1}$, as claimed.
\end{proof}

The bound in Theorem~\ref{th: time-bound-overlap-feedback} may seem very pessimistic.
However, the example in the next subsection shows that it is not too far off, in particular, it shows that the time until all the $reps$ fire can be exponential in $\lmax$.

\subsubsection{Lower bound}

Here we present an example of a concept hierarchy $\mathcal C$ and an input set $B$ for which the time for the $rep$ neurons for all the supported concepts to fire is exponential in $\lmax$.
This yields a lower bound, in Theorem~\ref{thm: big-lower-bound}.

The concept hierarchy $\mathcal C$ has levels $0,\ldots,\lmax$ as usual.
We assume here that $r_1 = r_2 = r$.
We assume that the overlap bound $o$ satisfies $o \cdot k \geq 2$, that is, the allowed overlap is at least $2$.
We take $f = 1$.

The network $\mathcal N$ embeds $\mathcal C$, as described earlier in this section.
As before, we assume that $\ell'_{max} = \lmax$, and
the threshold $\tau$ for the non-input nodes in the network is $r k$.
Now we assume that the weights are $1$ for both upward and downward edges between $reps$ of concepts in $C$, which is consistent with our choice of $f=1$ in the concept hierarchy.

We assume that hierarchy $\mathcal C$ has overlap only at one level---in the sets of children of level $2$ concepts.
The upper portion of $\mathcal C$, consisting of levels $2,\ldots \lmax$, is a tree, with no overlap among the sets $children(c)$, $3 \leq level(c) \leq \lmax$.
There is also no overlap among the sets of children of level $1$ concepts.

We order the children of each concept with $level \geq 3$ in some arbitrary order, left-to-right.
This orients the upper portion of the concept hierarchy, down to the level $2$ concepts.
Let $C'$ be the set of all the level $2$ concepts that are leftmost children of their parents.
Since there are $k^{\lmax -2}$ level $3$ concepts, it follows that $|C'| = k^{\lmax -2}$.
Number the elements of $C'$ in order left-to-right as $c_1,\ldots,c_{k_{\lmax -2}}$.
Also, for every concept $c_i$ in $C'$, order its $k$ children in some arbitrary order, left-to-right, and number them $1$ through $k$.

Now we describe the overlap between the sets of children of the level $2$ concepts $c_i$, $1 \leq i \leq k_{\lmax -2}$.
The first $k-1$ children of $c_1$ are unique to $c_1$, whereas its $k^{th}$ child is shared with $c_2$.
For $i =k_{\lmax -2}$, the last $k-1$ children of $c_i$ are unique to $c_i$, whereas its first child is shared with $c_{i-1}$.
For each other index $i$, the middle $k-2$ children of $c_i$ are unique to $c_i$, whereas its first child is shared with $c_{i-1}$, and its $k^{th}$ child is shared with $c_{i+1}$.
%[[[How do we know we have enough?  Need $k \geq 2$, which follows from $o k \geq 2$.]]]
%
There is no other sharing in $\mathcal C$.
%; in particular, as noted above, the sets of children of level $1$ concepts are disjoint.
%This completes the definition of $\mathcal C$.  

Next, we define the set $B$ of level $0$ concepts to be presented to the network.
$B$ consists of the following grandchildren of the level $2$ concepts in $C'$:

\begin{enumerate}
\item Grandchildren of $c_1$:
  \begin{enumerate}
  \item
    All the (level $0$) children of the children of $c_1$ numbered $1,\ldots, \lceil r k \rceil$,  and
  \item
    $\lceil r k \rceil - 1$ of the (level $0$) children of the $k^{th}$ child of $c_1$, which is also the first child of $c_2$.
  \end{enumerate}
\item Grandchildren of each $c_i$, $2 \leq i \leq k^{\lmax-2} - 1$:
  \begin{enumerate}
    \item
    $\lceil r k \rceil - 1$ of the (level $0$) children of the first child of $c_i$, which is also the $k^{th}$ child of $c_{i-1}$ (this has already been specified, just above),
  \item
    All the (level $0$) children of the children of $c_i$ numbered $2,\ldots, \lceil r k \rceil$, and
  \item
    $\lceil r k \rceil - 1$ of the (level $0$) children of the $k^{th}$ child of $c_i$, which is also the first child of $c_{i+1}$.
  \end{enumerate}
  \item Grandchildren of $c_i$, $i = k^{\lmax-2}$:
    \begin{enumerate}
            \item
    $\lceil r k \rceil - 1$ of the (level $0$) children of the first child of $c_i$, which is also the $k^{th}$ child of $c_{i-1}$ (this has already been specified, just above), and
      \item
    All the (level $0$) children of the children of $c_i$ numbered $2,\ldots, \lceil r k \rceil$.
  \end{enumerate}
\end{enumerate}

Figure~\ref{fig: overlap-lowerbound} illustrates a sample overlap pattern, for level $2$ neurons $c_1, c_2, c_3,...c_m$, where $m = k^{\lmax - 2}$.
Here we use $k = 4$, $r = 3/4$, and $o = 1/2$.

\vspace{-.3cm}
\begin{figure}[h]
\includegraphics[width = 6.5in]{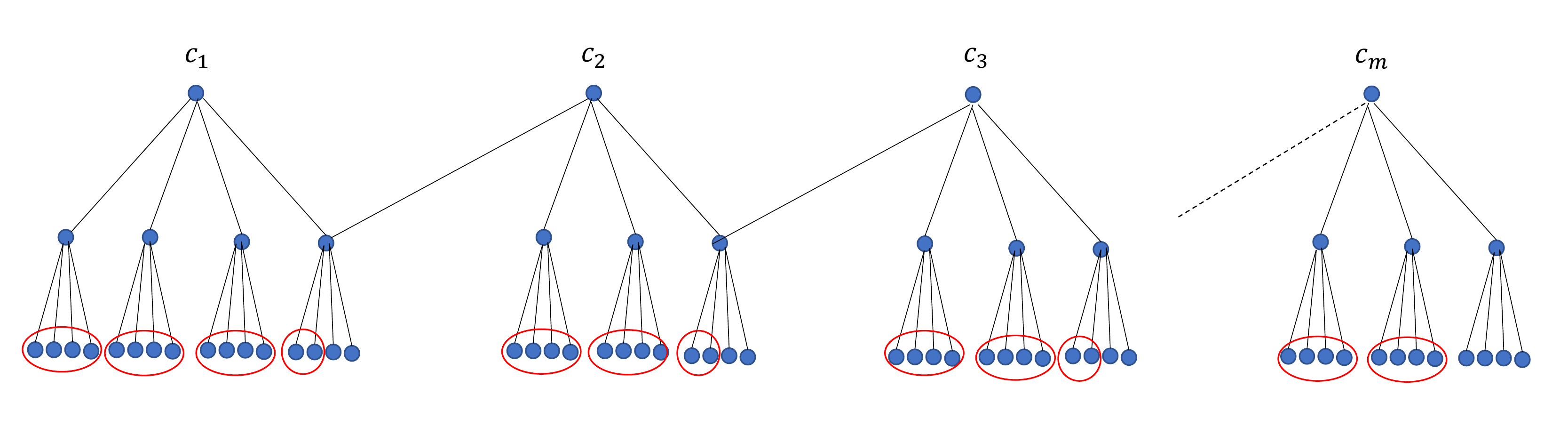}
\caption{Concept hierarchy with overlap, and input set.}
\label{fig: overlap-lowerbound}
\end{figure}
%\vspace{-3cm}

\begin{theorem}
\label{thm: big-lower-bound}
Assume $\mathcal C$ is the concept hierarchy defined above, and $\mathcal N$ is the network with feedback defined above, based on $\mathcal C$.
Let $B$ be the input set just defined, and assume that $B$ is presented at all times $\geq t$.
Then the time required for the $rep$ neurons for all concepts in $supp_{r,f}(B)$ to fire is at least $2 (k^{\lmax}-2)$.
\end{theorem}

\begin{proof}
The network behaves as follows:
\begin{itemize}
    \item Time $0$:  Exactly the $reps$ of concepts in $B$ fire.
    \item Time $1$:  The $reps$ of the (level $1$) children of $c_1$ numbered $1,\ldots,\lceil rk \rceil$ begin firing.
    Also, for every $c_i$, $2 \leq i \leq k^{\lmax-2}$, the $reps$ of the (level $1$) children numbered $2,\ldots,\lceil rk \rceil$ begin firing.
    This is because all of these neurons receive enough potential from the $reps$ of their (level $0$) children that fired at time $0$, to trigger firing at time $1$.
    No other neuron receives enough potential to begin firing at time $1$.
    \item Time $2$:
    Neuron $rep(c_1)$ begins firing, since it receives enough potential from the $reps$ of its first $\lceil r k \rceil$ children.
    No other neuron receives enough potential to begin firing at time $2$.
    \item Time $3$:
    Now that $rep(c_1)$ has begun firing, it begins contributing potential to the $reps$ of its children, via feedback edges.
    This potential is enough to trigger firing of the $rep$ of the (level $1$) $k^{th}$ child of $c_1$, when it is added to the potential from the $reps$ of that child's own level $0$ children.
    So, at time $3$, the $rep$ of the $k^{th}$ child of $c_1$ begins firing.
    No other neuron receives enough potential to begin firing at time $3$.
    \item Time $4$:
    The $k^{th}$ child of $c_1$ is also the first child of $c_2$.
    So its $rep$ contributes potential to $rep(c_2)$.
    This is enough to trigger firing of $rep(c_2)$, when added to the potential from the $reps$ of $c_2$'s already-firing children.
    So, at time $4$, $rep(c_2)$ begins firing.
    No other neuron receives enough potential to begin firing at time $4$.
    \item Time $5$:  
    Now that $rep(c_2)$ has begun firing, it contributes potential to the $reps$ of its children, via feedback edges.
    This is enough to trigger firing of the $rep$ of the (level $1$) $k^{th}$ child of $c_2$, when added to the potential from the $reps$ of that child's own level $0$ children.
    So, at time $5$, the $rep$ of the $k^{th}$ child of $c_2$ begins firing.
    No other neuron begins firing at time $3$.
    \item Time $6$:  In analogy with that happens at time $4$, neuron $rep(c_3)$ begins firing at time $6$, and no other neuron begins firing.
    \item $\ldots$
    \item Time $2 (k^{\lmax}-2)$:  Continuing in the same pattern, neuron $rep(c_{k^{\lmax-2}})$ begins firing at time $2 (k^{\lmax}-2)$.
\end{itemize}
Thus, the time to recognize concept $c_{k^{\lmax-2}}$ is exactly $2 (k^{\lmax}-2)$, as claimed.
%This is exponential in $\lmax$, as claimed.
\end{proof}

%%%%%%%%%%%%%%%%%%%%%%%%%%%%%%%%%%%%%%%%%
\subsection{Approximate Weights} 
\label{sec: uncertain-weights-2}

Now, as in Section~\ref{sec: uncertain-weights1}, we allow the weights to be specified only approximately.
We assume that $0 \leq w_1 \leq w_2$, as before.  Here we also assume that $b \geq 2$ and $k \geq 2$.

Let $u$ be any layer $\ell$ neuron and $v$ any layer $\ell+1$ neuron.
We define the weight for the upward edge $(u,v)$ by:
\[ weight(u,v) \in 
\begin{cases}
[w_1,w_2], & \text{ if } u,v \in R \text{ and } rep^{-1}(u) \in children(rep^{-1}(v)),
\\ 
[0,k^{\lmax+b}]  & \text{ otherwise.}
\end{cases}
\]
For the downward edge $(v,u)$, we define:
\[ weight(v,u) \in 
\begin{cases}
[f w_1, f w_2] & \text{ if } u,v \in R \text{ and } rep^{-1}(u) \in children(rep^{-1}(v)),
\\ 
[0,k^{\lmax+b}]  & \text{ otherwise.}
\end{cases} \]
As before, we set the threshold $\tau = \frac{(r_1  + r_2) k}{2}$.  We also use the trivial assumption that $\tau > \frac{1}{k^{b-2}}$.

For this case, we  prove:

\begin{theorem}
\label{th: overlap-ff-uncertain}
Assume $\mathcal C$ is any concept hierarchy satisfying the limited-overlap property, and $\mathcal N$ is the feed-forward network defined above, based on $\mathcal C$.
Assume that $\frac{(r_1+r_2)k}{2} \geq \frac{1}{k^{b-2}}$.
Suppose that $r_1$ and $r_2$ satisfy the following inequalities:
\begin{enumerate}
    \item $r_2 (2 w_1 - 1) \geq r_1$.
    \item $r_2 \geq r_1(2w_2 - 1) + \frac{2}{k^{b-1}}$.
\end{enumerate}
Then ${\cal N}$ $(r_1,r_2,f)$-\emph{recognizes} $\mathcal C$.
\end{theorem}

This theorem follows directly from the following lemmas:

\begin{lemma}
\label{lem: no-rep-no-fire-feedback-uncertain}
Assume $\mathcal C$ is any concept hierarchy satisfying the limited-overlap property, and $\mathcal N$ is the network defined above, based on $\mathcal C$.  
Assume that $\frac{(r_1 + r_2)k}{2} > \frac{1}{k^{b-2}}$.

Assume that $B \subseteq C_0$ is presented at all times $\geq t$. 
If $u$ is a neuron that fires at any time $t' \geq t$, then $u = rep(c)$ for some concept $c \in C$.  
\end{lemma}

\begin{proof}
By induction on the time $t' \geq t$, we show:
If $u$ is a neuron that fires at time $t'$, then $u = rep(c)$ for some concept $c \in C$.  
For the base case, $t'=t$, if $u$ fires at time $t$ then $u = rep(c)$ for some $c \in B$, by assumption.

For the inductive step, consider any neuron $u$ that fires at time $t'+1$, where $t' \geq t$.
Assume for contradiction that $u$ is not of the form $rep(c)$ for any $c \in C$.
Then the weight of each edge incoming to $u$ is at most $k^{\lmax+b}$.
By inductive hypothesis, the only incoming neighbors that fire at time $t'$ are $reps$ of concepts in $C$.
There are at most $k^{\lmax+1} + k^{\lmax-1}$ concepts at the two layers above and below $layer(u)$, hence at most $k^{\lmax+1} + k^{\lmax-1}$ neighbors of $u$ that fire at time $t'$, yielding a total incoming potential for $u$ for time $t'+1$ of at most
$\frac{k^{\lmax+1} + k^{\lmax-1}}{k^{\lmax+b}} = \frac{1}{k^{b-1}} + \frac{1}{k^{b+1}}$.
Since $k \geq 2$, this bound on potential is at most $\frac{1}{k^{b-2}}$.
Since the threshold $\tau = \frac{(r_1 + r_2)k}{2}$ is assumed to be strictly greater than $\frac{1}{k^{b-2}}$, $u$ does not receive enough incoming potential to meet the firing threshold for time $t'+1$.
\end{proof}

\begin{lemma}
\label{lem: fires-supported-feedback-uncertain} 
Assume $\mathcal C$ is any concept hierarchy satisfying the limited-overlap property, and $\mathcal N$ is the network with feedback as defined above, based on $\mathcal C$.
Assume that $\frac{(r_1+r_2)k}{2} > \frac{1}{k^{b-2}}$.
Also suppose that $r_1$ and $r_2$ satisfy the following inequalities:
\begin{enumerate}
    \item $r_2(2w_1-1) \geq r_1$.
    \item $r_2 \geq r_1(2w_2 - 1) + \frac{2}{k^{b-1}}$.
\end{enumerate}
Assume that $B \subseteq C_0$ is presented at all times $\geq t$.
If $c$ is any concept in $C$, then:
\begin{enumerate}
    \item If $c \in supp_{r_2,f}(B)$ then $rep(c)$ fires at some time $\geq t$.
    \item If $rep(c)$ fires at some time $\geq t$ then $c \in supp_{r_1,f}(B)$.
\end{enumerate}
\end{lemma}

\begin{proof}
The proof follows the general outline of the proof of Lemma~\ref{lem: fires-supported-feedback}, based on Lemma~\ref{lem: fires-supported-feedback-1}.
As in those results, the proof takes into account both the upward potential $upot$ and the downward potential $dpot$.
As before, we split the cases up and use two inductions based on time.
However, now the two inductions incorporate the treatment of variable weights used in the proof of Lemma~\ref{lem: fires-supported-ff-uncertain}.

\begin{enumerate}
    \item If $c \in S(t'-t)$ then $rep(c)$ fires at time $t'$.  Here the set $S(t'-t)$ is defined in terms of $supp_{r_2,f}(B)$.
  
We prove this using induction on $t'$, $t' \geq t$.  
For the base case, $t' = t$, the assumption that $c \in S(0)$ means that $c$ is in the input set $B$, which means that $rep(c)$ fires at time $t'$.
    
For the inductive step, assume that $t' \geq t$ and $c \in S((t'+1)-t)$.  If $level(c) = 0$ then again $c \in B$, so $c$ fires at time $t'$.
So assume that $level(c) \geq 1$.
Since $c \in S((t'+1)-t)$, we get that 
$|children(c) \ \cap \ S(t'-t)| + f \ |parents(c) \ \cap \ S(t'-t)| \geq r_2 k$.

By the inductive hypothesis, the $reps$ of all of these children and parents fire at time $t'$.
Therefore, the upward potential incoming to $rep(c)$ for time $t'+1$, $upot^{rep(c)}(t'+1)$, is at least $|children(c) \ \cap \ S(t'-t)| \  w_1$, and the downward potential incoming to $rep(c)$ for time $t'+1$, $dpot^{rep(c)}(t'+1)$, is at least $f \ |parents(c) \ \cap \ S(t'-t)| \ w_1$.
Adding these two potentials, we get that the total incoming potential to $rep(c)$ for time $t'+1$, $pot^{rep(c)}(t'+1)$, is at least 
$(|children(c) \ \cap \ S(t'-t)| + f \ |parents(c) \ \cap \ S(t'-t)|) \ w_1 \geq r_2 k w_1$.
This is at least $\frac{(r_1+r_2) k}{2}$, because of the assumption that $r_2 \ (2w_1 - 1) \geq r_1$.
So the incoming potential to $rep(c)$ for time $t'+1$ is enough to 
reach the firing threshold $\tau = \frac{(r_1+r_2) k}{2}$, so $rep(c)$ fires at time $t'+1$.

\item  If $rep(c)$ fires at time $t'$, then $c \in S(t'-t)$.  Here the set $S(t'-t)$ is defined in terms of $supp_{r_1,f}(B)$.  

We again use induction on $t'$, $t' \geq t$. 
For the base case, $t' = t$, the assumption that $rep(c)$ fires at time $t$ means that $c$ is in the input set $B$, hence $c \in S(0)$.
   
For the inductive step, assume that $rep(c)$ fires at time $t'+1$.
Then it must be that $pot^{rep(c)}(t'+1) = upot^{rep(c)}(t'+1) + dpot^{rep(c)}(t'+1)$ reaches the firing threshold $\tau = \frac{(r_1+r_2)k}{2}$ for $c$ to fire at time $t'+1$.   
Arguing as in the proof of Lemma~\ref{lem: no-rep-no-fire-feedback-uncertain}, the total incoming potential to $rep(c)$ from neurons at levels $level(c) - 1$ and $level(c)+1$ that are not $reps$ of children or parents of $c$ is at most $\frac{1}{k^{b-2}}$.
So the total incoming potential to $rep(c)$ from firing $reps$ of its children and parents must be at least $\frac{(r_1+r_2)k}{2} - \frac{1}{k^{b-2}}$.

By inductive hypothesis, all of these children and parents of $c$ are in $S(t'-t)$. 
Therefore, $(|children(c) \ \cap \ S(t'-t)| + f \ |parents(c) \ \cap \ S(t'-t)|) \ w_2  \geq \frac{(r_1+r_2)k}{2} - \frac{1}{k^{b-2}}$.
By the assumption that $r_2 \geq r_1(2w_2-1) + \frac{2}{k^{b-1}}$, we get that 
$|children(c) \ \cap \ S(t'-t)| + f \ |parents(c) \ \cap \ S(t'-t)| \geq r_1 k$.
(In more detail, let $E = |children(c) \cap S(t'-t)| + f |parents(c) \cap S(t'-t)|$.  So we know that $E w_2 \geq \frac{(r_1+r_2)k}{2} - \frac{1}{k^{b-2}}$.  Assume for contradiction that $E < r_1 k$.  Then $E w_2 < r_1 k w_2$.  But $r_1 k w_2 \leq (r1+r2)k/2 - 1/k^{b-2}$, because of the assumption that $r_2 \geq r_1(2w_2-1) + \frac{2}{k^{b-1}}$.
So that means that $E w_2 < (r1+r2)k/2 - 1/k^{b-2}$, which is a contradiction.)
Then by definition of $supp_{r_1,f}(B)$, we have that $c \in S((t'-t)+1) = S((t' + 1) - t)$, as needed.
\end{enumerate}
\end{proof}

The results of this section are also extendable to the case of scaled weights and thresholds, as in Section~\ref{sec: scaled}.

%%%%%%%%%%%%%%%%%%%%%%%%%%%%%%%%%%%%%%%%%%%%%%
\section{Learning Algorithms for Feed-Forward Networks}
\label{sec: learning-ff}

%Relate the previous results in LM21 to the scaled recognition results in Section 5.  We can probably show that what is achieved there in the end, in both the noise-free learning and noisy learning cases, fits the scaled definitions and results. Verify that this extends to overlap.

Now we address the question of how concept hierarchies (with and without overlap) can be learned in layered networks.
In this section, we consider learning in feed-forward networks, and in Section~\ref{sec: learning-feedback} we consider networks with feedback.

For feed-forward networks, we describe noise-free learning algorithms, which produce edge weights for the upward edges that suffice to support robust recognition.
These learning algorithms are adapted from the noise-free learning algorithm in~\cite{LM21}, and work for both tree hierarchies and general concept hierarchies.
We show that our new learning algorithms can be viewed as producing approximate, scaled weights as described in Section~\ref{sec: recognition-ff}, which serves to decompose the correctness proof for the learning algorithms.
We also discuss extensions to noisy learning.

\subsection{Tree Hierarchies}
\label{sec: learning-trees}

We begin with the case studied in~\cite{LM21}, tree hierarchies in feed-forward networks.

\subsubsection{Overview of previous noise-free learning results}

In~\cite{LM21}, we set the threshold $\tau$ for every neuron in layers $\geq 1$ to be 
$\tau = \frac{(r_1+r_2) \sqrt{k}} {2}$.
We assumed that the network starts in a state in which no neuron in layer $\geq 1$ is firing, and the weights on the incoming edges of all such neurons is $\frac{1}{k^{\lmax+1}}$.
We also assume a Winner-Take-All sub-network satisfying \autoref{as:WTA} below, which is responsible for engaging neurons at layers $\geq 1$ for learning.
These assumptions, together with the general model conventions for activation and learning using Oja's rule, determine how the network behaves when it is presented with a training schedule as in Definition~\ref{def: training-schedule}.

Our main result, for noise-free learning, is (paraphrased slightly)\footnote{We use $O$ notation here instead of giving actual constants.  We omit a technical assumption involving roundoffs.}:

\begin{theorem}[$(r_1,r_2)$-Learning]
\label{thm:noisefreelearning}
Let $\mathcal N$ be the network described above, with maximum layer
$\ell'_{max}$, and with learning rate $\eta = \frac{1}{4k}$.
Let $r_1, r_2$ be reals in $[0,1]$ with $r_1 < r_2$.
Let $\epsilon = \frac{r_2-r_1}{r_1+r_2}$.
Let $\mathcal C$ be any concept hierarchy, with maximum level $\lmax
\leq \ell'_{max}$.
Assume that the concepts in $\mathcal C$ are presented according to a
$\sigma$-bottom-up training schedule as defined in
Section~\ref{sec:prob-learning}, where $\sigma$ is
%$O\left( \frac{1}{\eta k} \left(\lmax \log(k) + \frac{1}{\epsilon}) \right) \right)$.
$O\left(\lmax \log(k) + \frac{1}{\epsilon}) \right)$.
Then $\mathcal N$ $(r_1,r_2)$-learns $\mathcal C$.
\end{theorem}
Specifically, we show that the weights for the edges from children to parents approach $\frac{1}{\sqrt{k}}$ in the limit, and the weights for the other edges approach $0$.

%[[[The assumption that sigma is O(something) is not really
%meaningful.  We should have a precise assumption here, with an actual
%constant.]]]

\subsubsection{The Winner-Take-All assumption} 

Theorem~\ref{thm:noisefreelearning} depends on Assumption~\ref{as:WTA} below, which hypothesizes a \emph{Winner-Take-All (WTA)} module with certain abstract properties.
This module is responsible for selecting a neuron to represent each new concept.
It puts the selected neuron in a state that prepares it to learn the concept, by setting the $engaged$ flag in that neuron to $1$.
It is also responsible for engaging the same neuron when the concept is presented in subsequent learning steps.

In more detail, while the network is being trained, example concepts
are ``shown'' to the network, according
to a $\sigma$-bottom-up schedule as described in Section~\ref{sec:prob-learning}.
We assume that, for every example concept $c$ that is shown, exactly
one neuron in the appropriate layer gets engaged; this layer is the
one with the same number as the level of $c$ in the concept hierarchy.
Furthermore, the neuron in that layer that is engaged is one that
has the largest incoming potential $pot^u$:

\begin{assumption}[Winner-Take-All Assumption]
\label{as:WTA}
If a level $\ell$ concept $c$ is shown at time $t$, then at time
$t+\ell$, exactly one neuron $u$ in layer $\ell$ has its $engaged$
state component equal to $1$, that is, $engaged^u(t+\ell) = 1$.
Moreover, $u$ is chosen so that $pot^u(t+\ell)$ is the highest
potential at time $t+\ell$ among all the layer $\ell$ neurons.
\end{assumption}

Thus, the WTA module selects the neuron to ``engage'' for learning. 
For a concept $c$ that is being shown for the first time, we showed that a new neuron is selected to represent $c$---one that has not previously been selected.
This is because, if a neuron has never been engaged in learning, its incoming weights are all equal to the initial weight $w = \frac{1}{k^{\lmax+1}}$, yielding a total incoming potential of $k w$.
On the other hand, those neurons in the same layer that have previously been engaged in learning have incoming weights for all of $c$'s children that are strictly less than the initial weight $w$, which yields a strictly smaller incoming potential. 
Also, for a concept $c$ that is being shown for a second or later time, we showed that the already-chosen representing neuron for $c$ is selected again.
This is because the total incoming potential for the previously-selected neuron is strictly greater than $k w$ (as a result of previous learning), whereas the potential for other neurons in the same layer is at most $k w$.

In a complete network for solving the learning problem, the WTA module would be implemented by a sub-network,
%It may use some method (like the one in~\cite{LMP19})
%to select the neuron. 
but we treated it abstractly in~\cite{LM21}, and we continue that approach in this paper.
%

%%%%%%%%%%
\subsubsection{Connections with our new results}
\label{sec: learning-weight-ranges-3}

Here we consider how we might use our scaled result in Section~\ref{sec: scaled} to decompose the proof of Theorem~\ref{thm:noisefreelearning} in~\cite{LM21}.
A large part of the proof in~\cite{LM21} consists of proving that the edge weights established as a result of a $\sigma$-bottom-up training schedule, for sufficiently large $\sigma$, are within certain bounds.
If these bounds match up with those in Section~\ref{sec: scaled}, we can use the results of that section to conclude that they are adequate for recognition.

The general definitions in Section~\ref{sec: scaled} use a threshold of $\frac{(r_1+r_2)k s}{2}$ and weights given by:
\[ weight(u,v) \in 
\begin{cases}
[w_1 s, w_2 s]  & \text{ if } u,v \in R \text{ and } rep^{-1}(u) \in children(rep^{-1}(v)),
\\ 
[0,\frac{1}{k^{\lmax+b}}]  & \text{ otherwise.}
\end{cases}
\]
For comparison, the lemmas in the proof of Theorem~\ref{thm:noisefreelearning} from~\cite{LM21} assert that, after a $\sigma$-bottom-up training schedule, we have (for $b \geq 2$):
% [[[It seems that we could use different values of $b$ here.]]]
\[ weight(u,v) \in 
\begin{cases}
[\frac{1}{(1+\epsilon)\sqrt{k}}, \frac{1}{\sqrt{k}}] & \text{ if } u,v \in R \text{ and } rep^{-1}(u) \in children(rep^{-1}(v)),
\\ 
[0,\frac{1}{k^{\lmax+b}}]  & \text{ otherwise.}
\end{cases}
\]
To make the results of~\cite{LM21} fit the constraints of Section~\ref{sec: scaled}, we can simply take $w_1 = \frac{1}{1+\epsilon}$, $w_2 = 1$, and the scaling factor $s = \frac{1}{\sqrt{k}}$.
The two constraints $r_2 (2w_1 - 1) \geq r_1$ and $r_2 \geq r_1 (2 w_2 - 1) + \frac{2}{k^b s}$
now translate into $r_2 (\frac{1 - \epsilon}{1+\epsilon}) \geq r_1$ and $r_2 \geq r_1 + \frac{2}{k^{b-\frac{1}{2}}}$.
The first of these, $r_2 (\frac{1 - \epsilon}{1+\epsilon}) \geq r_1$, follows from the assumption in~\cite{LM21} that $\epsilon = \frac{r_2 - r_1}{r_2+r_1}$.
The second inequality is similar to a roundoff assumption in~\cite{LM21} that we have omitted here.\footnote{In any case, we can made the decomposition work by adding our new, not-very-severe, inequality as an assumption.}
%
% [[[This does not quite match, but it's close, and I don't have more energy to try to reconcile, since it would probably involve revising the earlier paper.  Anyway, I think it's not important since we could just assume all of these inequalities, from both papers.]]]

%%%%%%%%%%
\subsubsection{Noisy learning}
\label{sec: noisy-learning}

In~\cite{LM21}, we extended our noise-free learning algorithm to the case of ``noisy learning''.
There, instead of presenting all leaves of a concept $c$ at every learning step, we presented only a subset of the leaves at each step.  This subset is defined recursively with respect to the hierarchical concept structure of $c$ and its descendants.
The subset varies, and is chosen randomly at each learning step.
Similar results hold as for the noise-free case, but with an increase in learning time.\footnote{
The extension to noisy learning is the main reason that we used the incremental Oja's rule. If the concepts were presented in a noise-free way, we could have allowed learning to occur all-at-once.
% [[[Right?]]]
}

The result about noisy learning in~\cite{LM21} assumes a parameter $p$ giving the fraction of each set of children that are shown; a larger value of $p$ yields a correspondingly shorter training time.
The target weight for learned edges is $\bar{w} = \frac{1}{\sqrt{pk+1-p}}$.
The threshold is $r_2k(\bar{w} - \delta)$, where
$\delta = \frac{(r_2-r_1)\bar{w}}{25}$.

The main result says that, after a certain time $\sigma$ (larger than the $\sigma$ used for noise-free learning) spent training for a tree concept hierarchy $\mathcal C$, with high probability, the resulting network achieves $(r_1,r_2)$-recognition for $\mathcal C$.
Here, a key lemma asserts that, with high probability, after time $\sigma$, the weights are as follows:

\[ weight(u,v) \in 
\begin{cases}
[\bar{w} - \delta, \bar{w} + \delta] & \text{ if } u,v \in R \text{ and } rep^{-1}(u) \in children(rep^{-1}(v)),
\\ 
[0,\frac{1}{k^{2\lmax}}]  & \text{ otherwise.}
\end{cases}
\]

To make these results fit the constraints of Section~\ref{sec: scaled}, it seems that we should modify the threshold slightly, by using the similar but simpler threshold $(\frac{(r_1+r_2)k}{2})\bar{w}$ in place of $r_2k(\bar{w} - \delta)$. 
The weights can remain the same as above, but we would express them in the equivalent form:
\[ weight(u,v) \in 
\begin{cases}
[(1-\frac{r_2-r_1}{25}) \bar{w}, (1+\frac{r_2-r_1}{25}) \bar{w}] & \text{ if } u,v \in R \text{ and } rep^{-1}(u) \in children(rep^{-1}(v)),
\\ 
[0,\frac{1}{k^{2\lmax}}]  & \text{ otherwise.}
\end{cases}
\]

Thus, we have scaled the basic threshold $\frac{(r_1+r_2)k}{2}$ by multiplying it by $\bar{w} = \frac{1}{\sqrt{pk+1-p}}$.
%and the limiting weight for an edge is now $\bar{w}$.
To make the results fit the constraints of Section~\ref{sec: scaled}, we can take
$s = \bar{w}$, $w_1 = 1 - \frac{r_2-r_1}{25}$, $w_2 = 1 + \frac{r_2-r_1}{25}$, and $b = \lmax$.
One can easily verify that the new thresholds still fulfill the requirements for recognition. First, 
we prove that a neuron fires when it should:
The weights on at least $r_2k$ incoming edges are $\bar{w} -\delta$. Hence, the incoming potential is at least $r_2 k (\bar{w} -\delta) $. Consider the difference of the incoming potential and the threshold
$r_2 k (\bar{w} -\delta) -\frac{(r_1+r_2)}{2}k \bar{w}
= \frac{r_2-r_1}{2} k \bar{w} - k\bar{w} \frac{r_2-r_1}{25} >0.
$
Hence, the neuron fires as required.
Second, we prove that a neuron will only fire if at least $r_1k$ incoming neurons fire.
Since all weights are at most $\bar{w}+\delta$, the incoming potential is at most $r_1 k (\bar{w} +\delta) $ and we claim this is strictly less than the threshold. 
To see this consider the difference between the threshold and the incoming potential: 
$\frac{(r_1+r_2)}{2}k \bar{w} - r_1 k (\bar{w} +\delta) = \frac{r_2-r_1}{2}k\bar{w}- r_1 k \delta \geq \frac{r_2-r_1}{2}k\bar{w}-  k\bar{w} \frac{r_2-r_1}{25} >0.
$ Hence, the neuron does not fire as required.

We leave it for future work to carry out all the detailed proof modifications.

%%%%%%%%%%%%%%%%%%%%%%%%%%%%%%%%%%%%%%%%%%%%%%
\subsection{General Concept Hierarchies}
\label{sec: learning-general}

The situation for general hierarchies, with limited overlap, in feed-forward networks is similar to that for tree hierarchies.
The same learning algorithm, based on Oja's rule, still works in the presence of overlap, with little modification to the proofs.
The only significant new issue to consider is how to choose an acceptable neuron to engage in learning, at each learning step.

We continue to encapsulate this choice within a separate WTA service.
As before, the WTA should always select an unused neuron (in the right layer) for a concept that is being shown for the first time.
And for subsequent times when the same concept is shown, the WTA should choose the same neuron as it did the first time.

\subsubsection{An issue with the previous approach}

Assumption~\ref{as:WTA}, which we used for tree hierarchies, no longer suffices.  For example, consider two concepts $c$ and $c'$ with $level(c') = level(c)$, and suppose that there is exactly one concept $d$ in the intersection $children(c) \cap children(c')$.
Suppose that concept $c$ has been fully learned, so a $rep(c)$ neuron has been defined, and then concept $c'$ is shown for the first time.
Then when $c'$ is first shown, $rep(c)$ will receive approximately $\frac{1}{\sqrt{k}}$ of total incoming potential, resulting from the firing of $rep(d)$.
On the other hand, any neuron that has not previously been engaged in learning will receive potential  $\frac{k}{k^{\lmax+1}} = \frac{1}{k^{\lmax}}$, based on $k$ neurons each with initial weight $\frac{1}{k^{\lmax+1}}$, which is smaller than $\frac{1}{\sqrt{k}}$.
Thus, Assumption~\ref{as:WTA} would select $rep(c)$ in preference to any unused neuron.

One might consider replacing Oja's learning rule with some other rule, to try to retain Assumption~\ref{as:WTA}, which works based just on comparing potentials.
%We discuss such an approach briefly in Section~\ref{sec: modified-learning-rule}.
Another approach, which we present here, is to use a ``smarter'' WTA, that is, to modify Assumption~\ref{as:WTA} so that it takes more information into account when engaging a neuron.

\subsubsection{Approach using a modified WTA assumption}
\label{sec: modified-WTA}

In the assumption below, $w$ denotes the initial weight,
$\frac{1}{k^{\lmax+1}}$.
$N_{\ell}$ denotes the set of layer $\ell$ neurons.
We make the trivial assumption that $o < 1$ for the noise-free case; for the noisy case, we strengthen that to $o < p$, where $p$ is the parameter indicating how many child concepts are chosen.

\begin{assumption}[Revised Winner-Take-All Assumption]
\label{as:WTA2}
If a level $\ell$ concept $c$ is shown at time $t$, then at time
$t+\ell$, exactly one neuron $u \in N_{\ell}$ has its $engaged$ state component equal to $1$, that is, $engaged^u(t+\ell) = 1$.
Moreover, $u$ is chosen so that $pot^u(t+\ell)$ is the highest
potential at time $t+\ell$ among the layer $\ell$ neurons that have strictly more than $o \cdot k$ incoming neighbors that contribute potential that is $\geq w$.\footnote{One might wonder whether such a choice is always possible, that is, whether the set of candidate neurons is guaranteed to be nonempty.  As we will see from the lemmas below, each concept engages only one neuron, leaving many in layer $\ell$ that never become engaged during the learning process.  These unengaged neurons keep their initial incoming weights of $w$.  Therefore, they satisfy the WTA requirement to be candidate neurons, provided that they have strictly more than $o \cdot k$ incoming neighbors that fire. 

But we know that, during the learning process, they have $k$ incoming neighbors that fire (in the noise-free case), which is $> o \cdot k$ by the assumption that $o < 1$. A similar argument holds for the noisy case, based on having at least $p k$ incoming neighbors that fire, which is $> o \cdot k$ by the assumption that $o < p$.}
\end{assumption}

Thus, we are assuming that the WTA module is ``smart enough'' to select the neuron to engaged based on a combination of two criteria:  First, it rules out any neuron that has just a few incoming neighbors that contribute potential $\geq w$.  This is intended to rule out neurons that have already started learning, but for a different concept.  Second, it uses the same criterion as in Assumption~\ref{as:WTA}, choosing the neuron with the highest potential from among the remaining candidate neurons.

We claim that using Assumption~\ref{as:WTA2} in the learning protocol yields appropriate choices for neurons to engage, as expressed by Lemma~\ref{lem: learning-2} below.
Showing these properties depends on a characterization of the incoming weights for a neuron $u \in N_{\ell}$ at any point during the learning protocol, as expressed by Lemma~\ref{lem: learning-1}.

\begin{lemma}
\label{lem: learning-1}
During execution of the learning protocol, at a point after any finite number of concept showings, the following properties hold:
\begin{enumerate}
\item
If $u$ has not previously been engaged for learning, then all of $u$'s incoming weights are equal to the initial weight $w$.
\item
If $u$ has been engaged for learning a concept $c$, and has never been engaged for learning any other concept, then all of $u$'s incoming weights for $reps$ of concepts in $children(c)$ are strictly greater than $w$, and all of its other incoming weights are strictly less than $w$.
\end{enumerate}
\end{lemma}
j
\begin{proof}
Property 1 is obvious---if a neuron is never engaged for learning, its incoming weights don't change.
Property 2 follows from Oja's rule.
\end{proof}

\begin{lemma}
\label{lem: learning-2}
During execution of the learning protocol, the following properties hold for any concept showing:
\begin{enumerate}
\item
If a concept $c$ is being shown for the first time, the neuron that gets engaged for learning $c$ is one that was not previously engaged.
\item
If a concept $c$ is being shown for the second or later time, the neuron that gets engaged for learning $c$ is the same one that was engaged when $c$ was shown for the first time.
\end{enumerate}
\end{lemma}

\begin{proof}
We prove Properties 1 and 2 together, by strong induction on the number $m$ of the concept showing.
%For the base case, $m=1$ so this is the first concept showing.
%Then Property 1 holds because the neuron that gets engaged was not previously engaged.  Property 2 is vacuously true.
% [[[We don't actually need a base case for strong induction.]]]
%
For the inductive step, suppose that concept $c$ is being shown at the $m^{th}$ concept showing.
By (strong) induction, we can see that, for each concept that was previously shown, the same neuron was engaged in all of its showings.
Therefore, the weights described in Lemma~\ref{lem: learning-1}, Property 2, hold for all neurons that have been engaged in showings $1,\ldots,m-1$.

\noindent\emph{Claim:}
Consider any neuron $u$ with $layer(u) = level(c)$ that was previously engaged for learning a different concept $c' \neq c$.
Then $u$ has at most $o \cdot k$ incoming neighbors that contribute potential to $u$ that is $\geq w$, and so, is not eligible for selection by the WTA.

\noindent \emph{Proof of Claim:}
Lemma~\ref{lem: learning-1}, Property 2, implies that all the incoming weights to neuron $u$ for $reps$ of concepts in $children(c')$ are strictly greater than $w$, and all of its other incoming weights are strictly less than $w$.  
Since $|children(c) \cap children(c')| \leq o \cdot k$, $u$ has at most $o \cdot k$ incoming neighbors that contribute potential that is $\geq w$, as claimed. \\
\noindent \emph{End of proof of Claim}

Now we prove Properties 1 and 2:
\begin{enumerate}
    \item If concept $c$ is being shown for the first time, the neuron $u$ that gets engaged for learning $c$ is one that was not previously engaged.  
    
    Assume for contradiction that the chosen neuron $u$ was previously engaged.
    Then it must have been for a different concept $c' \neq c$, since this is the first time $c$ is being shown.
    Then by the Claim, $u$ is not eligible for selection by the WTA.
    This is a contradiction.

\item If concept $c$ is being shown for the second or later time, the neuron $u$ that gets engaged for learning $c$ is the same one that was engaged when $c$ was shown for the first time.

Arguing as for Property 1, again using the Claim, we can see that $u$ cannot have been previously engaged for a concept $c' \neq c$.
So the only candidates for $u$ are neurons that were not previously engaged, as well as the (single) neuron that was previously engaged for $c$.
The given WTA rule chooses $u$ from among these candidate based on highest incoming potential.

For neurons that were not previously engaged, Lemma~\ref{lem: learning-1}, Property 1, implies that the incoming potential is exactly $k w$.  For the single neuron that was previously engaged for $c$, Lemma~\ref{lem: learning-1}, Property 2 implies that the incoming potential is strictly greater than $kw$.  So the WTA rule selects the previously-engaged neuron.
\end{enumerate}
\end{proof}

With the new WTA assumption, the learning analysis for general hierarchies follows the same pattern as the analysis for tree hierarchies in~\cite{LM21}, and yields the same time bound.

Implementing Assumption~\ref{as:WTA2} will require some additional mechanism, in addition to the mechanisms that are used to implement the basic WTA satisfying Assumption~\ref{as:WTA}.
Such a mechanism could serve as a pre-processing step, before the basic WTA.
The new mechanism could allow a layer $\ell$ neuron $u$ to fire (and somehow reflect its incoming potential) exactly if $u$ has strictly more than $o \cdot k$ incoming neighbors that contribute potential $\geq w$ to $u$.\footnote{
For instance, each layer $\ell-1$ neuron $v$ might have an outgoing edge to a special threshold element that fires exactly if the potential produced by $v$ on the edge $(v,u)$ is at least $w$, i.e., if $v$ fires and $weight(v,u) \geq w$.
Then another neuron associated with $u$ might collect all this firing information from all layer $\ell-1$ neurons $v$ and see if the number of firing neurons reaches the threshold $\lfloor o \cdot k \rfloor + 1$, which is equivalent to saying that the number of firing neurons is strictly greater than $o \cdot k$.
If this special neuron fires, it excites $u$ to act as an input to the basic WTA, but if it does not, $u$ should drop out of contention.}

\vspace{-.3cm}
\paragraph{Alternative modifications to the WTA:}
The modification in Section~\ref{sec: modified-WTA} allows the WTA to explicitly reject candidate neurons having too few incoming neighbors whose edge weight is greater than or equal to the initial weight $w$.  There might be other ways to modify the WTA.
For example, we might assume that the WTA can remember when it has already selected a neuron to represent a particular concept, and not select it again for a different concept.
This requires that the WTA have some way of remembering already-established associations between concepts and neurons.

This approach is inspired by the duplicate-avoidance mechanism in Step 3 of the learning algorithm in~\cite{HLMP}.
That algorithm incorporates a ``memory module'' that remembers already-established associations, and uses inhibition to prevent conflicting associations.
That work assumes an asymmetric network, and associates neurons with concepts in a predetermined sequential order.
Our setting here is different in that our network is symmetric, and the symmetry must be broken by the WTA.

Details of this alternative approach remains to be worked out.

% Each concept gets learned
%completely before going on to the next.  Small overlap is allowed.
%
%Our situation differs in that we have symmetry, which must be broken
%by the WTA.
%Also we may be looking at fine-grained interleaving, which the
%renaming paper doesn't consider.
%So I am not sure whether we can use any of the ideas from the paper.

%%%%%%%%%%%%%%%%%%%%%%%%%%%%%%%%%%%%%%%%%%%%%%
\vspace{-.3cm}
\paragraph{Modified learning rules:}

Another approach might be to try to modify Oja's learning rule while keeping the original WTA assumption, Assumption~\ref{as:WTA}.
As noted at the start of this section, even when two concepts $c$ and $c'$ share only one child, if $c$ is fully learned, and $c'$ is being shown for the first time, the potential incoming to $rep(c)$ from the $reps$ of $children(c')$ will be higher than that incoming to any neuron that has not yet been engaged in learning; that is, it is higher than $k w$, where $w$ is the initial weight.
Similar situations arise when the concept $c$ is only partially learned.

The key seems to be to keep the potential incoming to $rep(c)$ from the $reps$ of children of $c'$ strictly below $kw$, throughout the process of learning $c$.
But that might not be reasonable, since we assume that $w$ is very small.

\section{Learning Algorithms for Networks with Feedback}
\label{sec: learning-feedback}

Now we consider how concept hierarchies, with and without overlap, can be learned in layered networks with feedback.
The learning algorithms described in Section~\ref{sec: learning-ff} set the weights on the directed edges from each layer $\ell$ to the next higher layer $\ell+1$, that is, the ``upward'' edges.
Now the learning algorithm must also set the weights on the directed edges from each layer $\ell$ to the next lower layer $\ell-1$, i.e., the ``downward'' edges.

One reasonable approach is to separate matters, first learning the weights on the upward edges and then the weights on the downward edges.  Fortunately, we can rely on Lemma~\ref{lem: total-support}, which says that, if $c$ is any concept in a concept hierarchy $\mathcal C$, then $c \in supp_{1} (leaves(c))$.
That is, any concept is supported based only on its descendants, without any help from its parents.
This lemma implies that learning of upward edges can proceed bottom-up, as in Section~\ref{sec: learning-ff}.
%This is true for general concept hierarchies, not just tree hierarchies.
We give some details below.

\subsection{Noise-Free Learning}

As in Section~\ref{sec: learning-ff}, we assume that the threshold is $\frac{(r_1+r_2)\sqrt{k}}{2}$, the initial weight for each upward edge is $w = \frac{1}{k^{\lmax+1}}$, and $\epsilon = \frac{r_2 - r_1}{r_1+r_2}$.
Here we also assume that the initial weight for each downward edge is $w$.\footnote{We are omitting mention here of some trivial technical assumptions, like small lower bounds on $\tau$.}
We assume that the network starts in a state in which no neuron in layer $\geq 1$ is firing.
%We make the (trivial) assumption that $r_1 k > 1$.

Our main result, for noise-free learning, is:

\begin{theorem}
\label{thm: learning-feedback}
Let $\mathcal N$ be the network defined in this section, with maximum layer $\ell'_{max}$, and with learning rate $\eta = \frac{1}{4k}$.
Let $r_1, r_2$ be reals in $[0,1]$ with $r_1 \leq r_2$.
Let $\epsilon = \frac{r_2-r_1}{r_1+r_2}$.

Let $\mathcal C$ be any concept hierarchy, with maximum level $\lmax
\leq \ell'_{max}$.
Assume that the algorithm described in this section is executed:  
On the first pass, the concepts in $\mathcal C$ are presented according to a
$\sigma$-bottom-up presentation schedule, where $\sigma$ is
$O\left(\lmax \log(k) + \frac{1}{\epsilon}) \right)$.
The second pass is as described in Section~\ref{sec: second-pass}.
Then $\mathcal N$ $(r_1,r_2,f)$-learns $\mathcal C$.
\end{theorem}

\subsubsection{First learning pass}

As a first pass, we carry out the learning protocol from Section~\ref{sec: learning-ff} for all the concepts in the concept hierarchy $\mathcal C$, working bottom-up. 
Learning each concept involves applying Oja's rule for that concept, for enough steps to ensure that the weights of the upward edges end up within the bounds described in Section~\ref{sec: scaled}.

Consider the network after the first pass, when the weights of all the upward edges have reached their final values.  
At that point, we have that the network $(r_1,r_2)$-recognizes the given concept hierarchy $\mathcal C$, as described in Section~\ref{sec: learning-ff}.
Moreover, we obtain:
\begin{lemma}
\label{lem: phase1}
The weights of the edges after the completion of the first learning pass are as follows:
\begin{enumerate}
    \item  The weights of the upward edges from $reps$ of children to $reps$ of their parents are in the range $[\frac{1}{(1+\epsilon)\sqrt{k}}, \frac{1}{\sqrt{k}})]$, and the weights of the other upward edges are in the range $[0,\frac{1}{2^{lmax+b}}]$.
    \item  The weights of all downward edges are $w = \frac{1}{k^{\lmax+1}}$.
\end{enumerate}
\end{lemma}

As a consequence of these weight settings, we can prove the following about the network resulting from the first pass:

\begin{lemma}
\label{lem: phase1b}
The following properties hold of the network that results from the completion of the first learning pass:
\begin{enumerate}
    \item Suppose $c$ is any concept in $C$.  
    Suppose that $c$ is shown (that is, the set $leaves(c)$ is presented) at time $t$, and no inputs fire at any other times.
    Then $rep(c)$ fires at time $t + level(c) = t + layer(rep(c))$, and does not fire at any earlier time.
    \item Suppose $c$ is any concept in $C$.  
    Suppose that $c$ is shown at time $t$, and no inputs fire at any other times.
    Suppose $c'$ is any other concept in $C$ with $level(c') = level(c)$.  Then $rep(c')$ does not fire at time $t+level(c)$.
    \item Suppose that $u$ is a neuron in the network that is not a $rep$ of any concept in $C$.
    Suppose that precisely the set of level $0$ concepts in $C$ is presented at time $t$, and no inputs fire at any other times.
    Then neuron $u$ never fires.
\end{enumerate}
\end{lemma}

\begin{proof}
Property 1 follows from the analysis in~\cite{LM21}, plus the fact that $level(c)$ is the time it takes to propagate a wave of firing from the inputs to $layer(rep(c)$.
Property 2 can be proved by induction on $level(c)$, using the limited-overlap property.
Property 3 can be proved by induction on the time $t' \geq t$, analogously to Lemma~\ref{lem: no-rep-no-fire-feedback-uncertain}.

\end{proof}

\subsubsection{Second learning pass}
\label{sec: second-pass}

The second pass sets all the weights for the downward edges.
Here, to be simple, we set the weight of each edge to its final value in one learning step, rather than proceeding incrementally.\footnote{This seems reasonable since we are not considering noise during this second pass.  Of course, that might be interesting to consider at some point.}
We aim to set the weights of all the ``important'' downward edges, that is, those that connect the $rep$ of a concept to the $rep$ of any of its children, to $\frac{f}{\sqrt{k}}$, and the weights of all other downward edges to $0$.

We first set the weights on the ``important'' downward edges.
For this, we proceed level by level, from $1$ to $\lmax$.
The purpose of the processing for level $\ell$ is to set the weights on all the ``important'' downward edges from layer $\ell$ to layer $\ell-1$ to $\frac{f}{\sqrt{k}}$, while leaving the weights of the other downward edges equal to the initial weight $w$.

For each particular level $\ell$, we proceed sequentially through the level $\ell$ concepts in $C$, one at a time, in any order.
For each such level $\ell$ concept $c$, we carry out the following three steps:
\begin{enumerate}
    \item  Show concept $c$, that is, present the set $leaves(c)$, at some time $t$.  By Theorem~\ref{lem: phase1}, the $reps$ of all children of $c$ fire at time $t + \ell - 1$, and  $rep(c)$ fires at time $t+\ell$.
    \item  Engage all the layer $\ell-1$ neurons to learn their incoming $dweights$ at time $t+\ell+1$, by setting their $dgaged$ flags.
    \item  \emph{Learning rule:}  At time $t+\ell+1$, each $dgaged$ layer $\ell-1$ neuron $u$ that fired at time $t+\ell-1$ sets the weights of any incoming edges from layer $\ell$ neurons that fired at time $t+\ell$ (and hence contributed potential to $u$) to $\frac{f}{\sqrt{k}}$.  
    Neuron $u$ does not modify the weights of other incoming edges.
\end{enumerate}

Note that, in Step 3, each neuron $u$ that fired at time $t+\ell-1$ will set the weight of at most one incoming downward edge to $\frac{f}{\sqrt{k}}$; this is the edge from $rep(c)$, in case $u$ is the $rep$ of a child of $c$. 

Also note that $u$ does not reduce the weights of other incoming downward edges during this learning step.  This is to allow $u$ to receive  potential from other layer $\ell$ neurons when those concepts are processed. This is important because $u$ may represent a concept with multiple parents, and must be able to receive potential from all parents when they are processed.

Finally, note that, to implement this learning rule, we need some mechanism to engage the right neurons at the right times.
For now, we just treat this abstractly, as we did for learning in feed-forward networks in Section~\ref{sec: learning-ff}.

At the end of the second pass, each neuron $u$ resets the weights of all of its incoming downward edges that still have the initial weight $w$, to $0$.  The neurons can all do this in parallel.

\begin{lemma}
\label{lem: phase2}
The weights of the edges after the completion of the second learning pass are as follows:
\begin{enumerate}
    \item  The weights of the upward edges from $reps$ of children to $reps$ of their parents are in the range $[\frac{1}{(1+\epsilon)\sqrt{k}}, \frac{1}{\sqrt{k}})]$, and the weights of the other upward edges are in the range $[0,\frac{1}{2^{lmax+b}}]$.
    \item  The weights of the downward edges from $reps$ of parents to $reps$ of their children are $\frac{f}{\sqrt{k}}$, and the weights of the other downward edges are $0$.
\end{enumerate}
\end{lemma}

\begin{proof}
Property 1 follows from Lemma~\ref{lem: phase1} and the fact that the weights of the upward edges are unchanged during the second pass.

For Property 2, the second pass is designed to set precisely the claimed weights.
This depends on the neurons firing at the expected times.
This follows from Lemma~\ref{lem: phase1b}, once we note that the three claims in that lemma remain true throughout the second pass.
(The first two properties depend on upward weights only, which do not change during the second pass.  Property 3 follows because only $rep$ neurons have their incoming weights changing during the second pass.)

\end{proof}

With these weights, we can now prove the main theorem:

\begin{proof}
(Of Theorem~\ref{thm: learning-feedback}:)
We use a scaled version of Theorem~\ref{th: overlap-ff-uncertain}.
Here we use $w_1 = \frac{1}{1+\epsilon}$, $w_2 = 1$, and a scaling factor $s = \frac{1}{\sqrt{k}}$.\footnote{The scaled case isn't actually worked out in Section 6 but should follow as a natural extension of the un-scaled results.}

\end{proof}

\subsection{Noisy Learning}

We can extend the results of the previous section to allow noisy learning in the first pass.
For this, we use a threshold of $(\frac{(r_1+r_2)k}{2})\bar{w}$ and retain the initial edge weights of $w = \frac{1}{k^{\lmax + 1}}$.
We define $w_1 =  1-\frac{r_2-r_1}{25}$, $w_2 = 1-\frac{r_2-r_1}{25}$, and the scaling factor $s$ to be $\bar{w} = \frac{1}{\sqrt{pk + 1 - p}}$.

The ideas are analogous to the noise-free case.  The differences are:
\begin{enumerate}
    \item  The first phase continues long enough to complete training for the weights of the upward edges using Oja's rule.
    \item The weights of the upward edges from $reps$ of children to $reps$ of their parents are in the range $[(1-\frac{r_2-r_1}{25}) \bar{w}, (1+\frac{r_2-r_1}{25}) \bar{w}]$, and the weights of the other upward edges are in the range $[0,\frac{1}{k^{2\lmax}}]$.
    \item  The weights of the downward edges are set to $f \bar{w}$.
\end{enumerate}

With these changes, we can obtain a theorem similar to Theorem~\ref{thm: learning-feedback} but with a larger training time, yielding $(r_1,r_2,f)$-learning of $\mathcal C$.

%%%%%%%%%%%%%%%%%%%%%%%%%%%%%%%%%%%%%%%%%%%%%%%%%%%%%%%%%%%%%%%%%%%%%%%%%%%%%%%%%%

%%%%%%%%%%%%%%%%%%%%%%%%%%%%%%%%%%%%%%%%%%%%%%%%%%%%%%%%%%%%%%%%%%%%%%%%%
\section{Conclusions}
\label{sec: conclusions}

In this paper, we have continued our study from~\cite{LM21}, of how hierarchically structured concepts might be represented in brain-like neural networks, how these representations might be used to recognize the concepts, and how these representations might be learned.
In~\cite{LM21}, we considered only tree-structured concepts and feed-forward layered networks.
Here we have made two important extensions:  our data model may include limited overlap between the sets of children of different concepts, and our networks may contain feedback edges.
We have considered algorithms for recognition and algorithms for learning.
We think that the most interesting theoretical contributions in the paper are:
\begin{enumerate}
    \item  Formal definitions for concept hierarchies with overlap, and networks with feedback.
    \item  Analysis for time requirements for robust recognition; this is especially interesting for general concept hierarchies in networks with feedback. 
    \item  Extensions of the recognition results to allow approximate edge weights and scaling.
    \item  The handling of Winner-Take-All mechanisms during learning, in the presence of overlap.
    \item  Reformulation of learning behavior in terms of achieving certain ranges of edge weights.
    \item  A simple strategy for learning bidirectional weights.
\end{enumerate}  
On the practical side, for computer vision, our recognition algorithms may provide plausible explanations for how objects in complex scenes can be identified, using both upward and downward information flow in a neural network.

There are many possible directions for extending the work.  For example:

\vspace{-.3cm}
\paragraph{Concept recognition:}
It would be interesting to study \emph{recognition behavior after partial learning}. 
The aim of the learning process is to increase weights sufficiently to guarantee recognition of a concept when partial information about its leaves is presented.
Initially, even showing all the leaves of a concept $c$ should not be enough to induce $rep(c)$ to fire, since the initial weights are very low.
At some point during the learning process, after the weights increase sufficiently, presenting all the leaves of $c$ will guarantee that $rep(c)$ fires. 
As learning continues, fewer and fewer of the leaves will be needed to guarantee firing.
It would be interesting to quantify the relationship between amount of learning and the number of leaves needed for recognition.
% [[[Of course, this "number of leaves" is defined recursively, not in term of a strict number of leaves.]]]

Also, the definition of robustness that we have used in this paper involves just omitting some inputs.
In would be interesting to also consider other types of noise, such as \emph{adding extraneous inputs}.
How well do our recognition algorithms handle this type of noise?

Another type of noise arises if we replace the deterministic threshold elements with neurons that fire stochastically, based on some type of sigmoid function of the incoming potential.  
How well do our recognition algorithms handle this type of noise?
Some initial ideas in this direction appear in Appendix~\ref{sec: stochastic}, but more work is needed.

\vspace{-.3cm}
\paragraph{Learning of concept hierarchies:}
Our learning algorithms depend heavily on Winner-Take-All subnetworks.  We have treated these abstractly in this paper, by giving formal assumptions about their required behavior.
It remains to \emph{develop and analyze networks implementing the Winner-Take-All assumptions}.

Another interesting issue involves possible \emph{flexibility in the order of learning concepts}.  
In our algorithms, incomparable concepts can be learned in any order, but children must be completely learned before we start to learn their parents.  We might also consider some interleaving in learning children and parents. 
Specifically, in order to determine $rep(c)$ for a concept $c$, according to our learning algorithms, we would need for the $reps$ of all of $c$'s children to be already determined, and for the children to be learned sufficiently so that their $reps$ can be made to fire by presenting ``enough'' of their leaves.
But this does not mean that the child concepts must be completely learned, just that they have been learned sufficiently that it is possible to make them fire (say, when all, or almost all, of their leaves are presented).  This suggests that it is possible to allow some interleaving of the learning steps for children and parents. This remains to be worked out.

Another issue involves \emph{noise in the learning process}.
In Sections~\ref{sec: learning-ff} and~\ref{sec: learning-feedback}, we have outlined results for noisy learning of weights of upward edges, in the various cases studied in this paper, but full details remain to be worked out.
The approach should be analogous to that in~\cite{LM21}, based on presenting randomly-chosen subsets of the leaves of a concept being learned.
The key here should be to articulate simple lemmas about achieving approximate weights with high probability.
It also remains to consider noise in the learning process for weights of downward edges.

Finally, our work on learning of weights of upward edges has so far relied on Oja's learning rule.
It would be interesting to consider \emph{different learning rules} as well, comparing the guarantees that are provided by different rules, with respect to speed of learning and robustness to noise during the learning process.

%Note that noise during learning is the main reason one uses incremental learning rules such as Oja's rule; without noise, one could modify weights of incoming edges for a neuron all at once, rather than in small increments.
%It might be interesting to consider all-at-once learning rules for noise-free learning.
% [[[Is this interesting?  Consider...]]]

\vspace{-.3cm}
\paragraph{Different data models, different network assumptions, different representations:}
One can consider many variations on our assumptions.
For example, what is the impact of loosening the very rigid assumptions about the shape of concept hierarchies?
What happens to the results if we have limited connections between layers, rather than all-to-all connections?  Such connections might be randomly chosen, as in~\cite{Valiant}.
Also, we have been considering a simplified representation model, in which each concept is represented by precisely one neuron; can the results be extended the to accommodate more elaborate representations?

%Relate to Minsky's k-lines?

\vspace{-.3cm}
\paragraph{Experimental work in computer vision:}
Finally, it would be interesting to try to devise experiments in computer vision that would reflect some of our theoretical results.
For example, can the high-latency recognition behavior that we identified in Section~\ref{sec: time-bounds-recognition-overlap-feedback}, involving extensive information flow up and down the hierarchy, be exhibited during recognition of visual scenes?

\subsubsection{Acknowledgments}
We thank Cristina Gava for her suggestions.
This work was supported in part by the National Science Foundation under awards CCF-1810758, CCF-2139936, and CCF-2003830, as well as by EPSRC EP/W005573/1.

\bibliography{OF}
\appendix
%%%%%%%%%%%%%%%%%%%%%%%%%%%%%%%%%%%%%%%%

\section{Example:  An Italian Restaurant Catering Menu}
\label{app: restaurant}

Here is an example to illustrate some of the ideas of Section~\ref{sec:datamodel}.  The example represents the catering menu of an Italian restaurant, called ``Tutto Italiano'', which has branches in Boston and London.  The catering menu consists of four regional meals, each consisting of four dishes.  Each dish has four main ingredients.

The four meals correspond to the following regions of Italy:  Emilia-Romagna (Bologna, Parma), Campania (Naples), Sicilia (Palermo, Catania), and Toscana (Florence, Pisa).

In terms of our model of Section~\ref{sec:datamodel}, the maximum level $\lmax$ is $2$.  The level $2$ concepts, $C_2$, are the meals, the level $1$ concepts, $C_1$, are the individual dishes, and the level $0$ concepts, $C_0$, are the ingredients.  We define $k=4$, $r_1 = r_2 = r = 3/4$, $o = 1/2$, and $f=1$.
Using $r = 3/4$ means that three ingredients should be enough to identify a dish.
Using $o = 1/2$ means that a dish could have two ingredients that are shared with other dishes, in addition to two unique ingredients.
The set $C$ of concepts includes all of the ingredients, dishes, and meals.

%We could also have two menus sharing a dish or two, although I haven't done this.
%If we want to do this, we can let struffoli (pignolata) be the dessert in both Naples and Sicily.  But it would be a shame to omit cannolis.

\subsection{The catering menu}\label{sec:Italy}

We allow the following ingredients to be shared, because they are so common in Italian cooking: 
tomatoes, onions, breadcrumbs, parmesan cheese, garlic, olive oil, lemon, and salt.
Every other ingredient is unique to one dish.
Every recipe has at least two unique ingredients.
% \footnote{Please excuse the mixture of Italian and English in naming the dishes and ingredients.}

\paragraph{Emilia-Romagna meal:}
This consists of the following four dishes, with the listed ingredients:
\begin{itemize}
    \item  \emph{Pasta Bolognese:} tagliatelle, ground beef, tomatoes, parmesan cheese.
    %(also nutmeg, milk)
    \item  \emph{Cotoletta di vitello alla Bolognese:}  veal cutlets, breadcrumbs, prosciutto, parmesan cheese.
    \item  \emph{Insalata di radicchio:}  radicchio, goat cheese, speck, balsamic vinegar.
    \item \emph{Zuppa Inglese:}  ladyfingers, custard, Alchermes liqueur, cocoa powder.
\end{itemize}

\paragraph{Campania meal:}
\begin{itemize}
\item \emph{Pizza Margherita:}  pizza dough, tomatoes, mozzarella, basil.
%(olive oil).
\item \emph{Spaghetti alla puttanesca:}  spaghetti, olives, tomatoes, anchovies.
% (chili peppers, garlic, parmesan)
\item \emph{Acqua pazza:}  cod, fennel, tomato, chili peppers.
\item \emph{Struffoli:}  dough balls, honey, almonds, colored sprinkles.
%(candied fruits)
\end{itemize}
\paragraph{Sicilia meal:}
\begin{itemize}
\item
\emph{Carciofi al forno:}  artichokes, lemon, pancetta, breadcrumbs.
%(pecorino)
\item
\emph{Pasta e cavolfiore:}  penne, cauliflower, raisins, pecorino cheese.
%(SAFFRON, PINE NUTS)
\item \emph{Pesce spada:}  swordfish, capers, butter, olive oil.
\item \emph{Cannoli:}  cannoli shells, ricotta, sugar, pistachios.
%(shaved chocolate)
\end{itemize}

\paragraph{Toscana meal:}
\begin{itemize}
    \item \emph{Ribollita:}  cannellini beans, carrot, onion, olive oil.
    \item \emph{Risotto al Chianti:}  arborio rice, chianti, onion, celery.
    \item \emph{Bistecca Fiorentina:}  steak, %(chianina), 
    spinach, rosemary, 
    %(or mixed herbs), 
    olive oil.
    \item \emph{Pesche con Amaretti:}  peaches, amaretti biscuits, marsala, lemon.
\end{itemize}

\begin{figure}
\includegraphics[width=\textwidth]{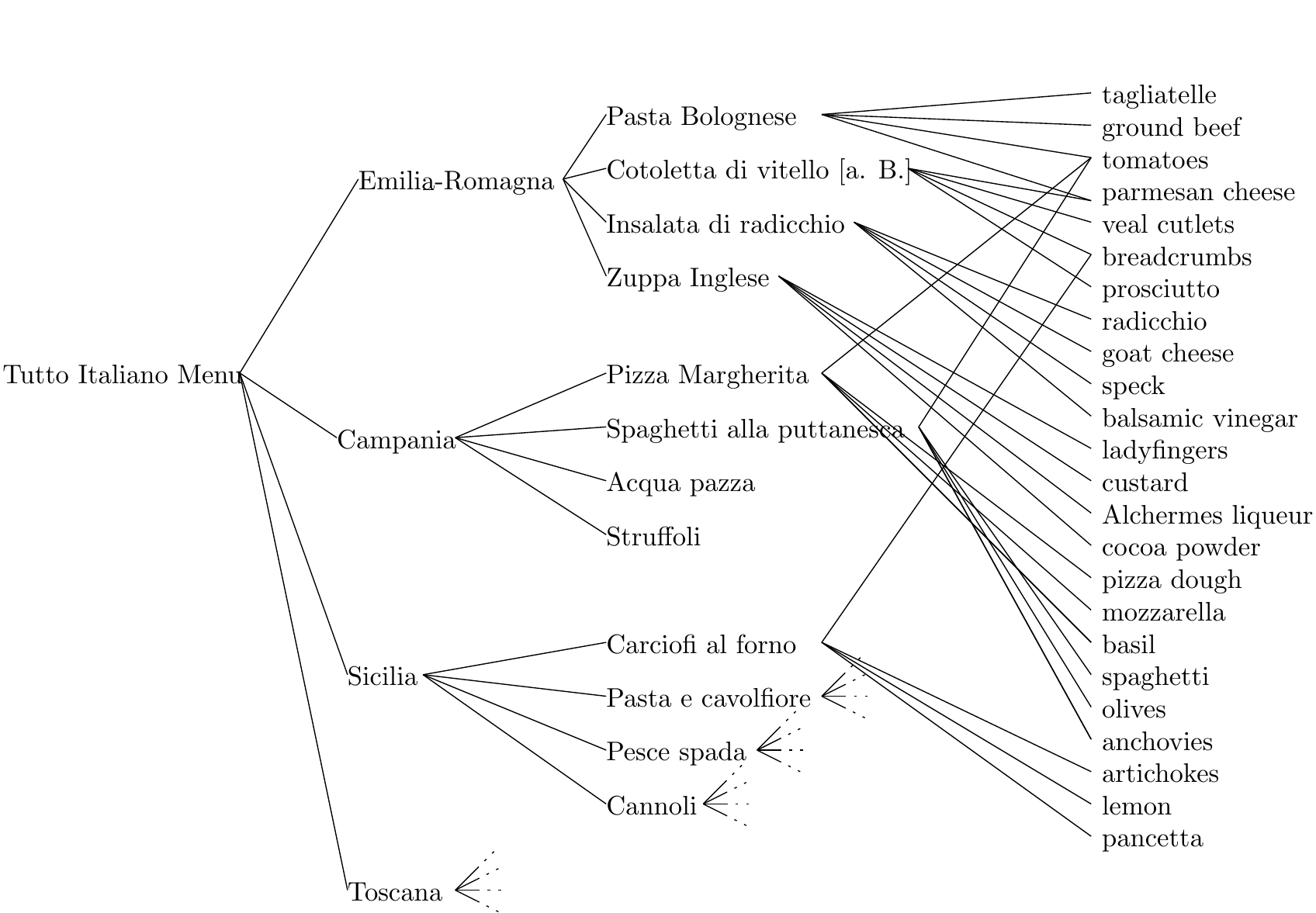}
\caption{A graphical representation of the catering menu example presented in Section~\ref{sec:Italy}.}
\end{figure}

\subsection{Support}

We illustrate the $supp_r$ definition for $r = 3/4$ and the $supp_{r,f}$ definition for $r = 3/4$ amd $f = 1$.
In both cases, we let $B$ be the following set of ingredients: 
\emph{artichokes, basil, breadcrumbs, butter, cannellini beans, capers, carrot, cauliflower, chili pepper, cod, lemon, olive oil, onion, parmesan cheese, pecorino cheese, pistachios, ricotta, spinach, steak, sugar, tomatoes.} 

For $supp_{3/4}(B)$, we follow the construction in Definition~\ref{def: supp1}.  We get that $S(0)$ is the set $B$ of ingredients listed above.  $S(1)$ consists of the following dishes:  \emph{acqua pazza, baked artichokes, pesce spada, cannoli, ribollita soup, and bistecca Fiorentina,} because $B$ contains at least three of the ingredients of each of these dishes. $S(2)$ consists of just the \emph{Sicilia} meal, because $S(1)$ contains three of its dishes.  The set $supp_{3/4}(B)$ is the union $S(0) \cup S(1) \cup S(2)$.

For $supp_{3/4,1}(B)$, we follow Definition~\ref{def: supp2}.  The key difference is that the set $supp_{3/4,1}(B)$ includes, in addition to all of the elements of $supp_{3/4}(B)$, the new dish \emph{cauliflower pasta}.  This is included based on the presence in $B$ of two of its ingredients, \emph{cauliflower} and \emph{pecorino chese}, together with the presence in $supp_{3/4}(B)$ of the \emph{Sicilia} meal.
Technically, $S(1,1)$ includes \emph{baked artichokes}, \emph{Pesce spada}, and \emph{cannoli} (as well as other dishes), $S(2,2)$ includes \emph{Sicilia}, and S(1,3) includes \emph{cauliflower pasta}.

How might we interpret these support definitions, for this example?  One might imagine that a visitor to the restaurant kitchen sees a large set of ingredients strewn haphazardly around the kitchen counter, and tries to deduce which dishes and which meals are being prepared. 
Some of the ingredients for a dish might be missing from the counter.
The visitor might try to do this by computing either $supp_{3/4}(B)$ or $supp_{3/4,1}(B)$, where $B$ is the set of observed ingredients.
Either of these provides a likely set of dishes and meals, by filling in some missing ingredients.

%Recognition follows the support definition precisely.

%%%%%%%%%%%%%%%%%%%%%%%%%%%%%%%%%%%%%%%%%%%%%%%%%%%%%%%%%%%%%%%%%%%%%%%%%%%%%%%%%%%%%%%%%%
\section{Recognition in Networks with Stochastic Firing}
\label{sec: stochastic}

Another type of uncertainty, besides approximate weights, arises when neuron firing is determined stochastically, for example, using a standard sigmoid function.  See~\cite{LynchMusco-arxiv21} for an example of a model that uses this strategy.
In this case, we cannot make absolute claims about recognition, but we would like to assert that correct recognition occurs with high probability.
Here we consider this type of uncertainty in the situation where weights are exactly $1$ or $0$, as in Section~\ref{sec: recognition-ff-exact}.  Extension to allow approximate weights, as well as to networks with feedback, is left for future work.

Following~\cite{LynchMusco-arxiv21}, we assume that the potential incoming to a neuron $u$, $pot^u$, is adjusted by subtracting a real-valued $bias$ value, and the resulting adjusted potential, $x$, is fed into a standard sigmoid function with temperature parameter $\lambda$, in order to determine the firing probability $p$.
Specifically, we have:
\[
p(x) = \frac{1}{1 + e^{-x/\lambda}},
\]
where $x = pot^u - bias$.
%Here we would like $bias$ to be a sufficiently large positive number to shift the sigmoid sufficiently far to the right, see below.

Let $\delta$ be a small target failure probability.
In terms of our usual parameters $n$, $f$, $k$, and $\lmax$, and the new parameters $\lambda$ and $\delta$, our goal is to determine values of $r_1$ and $r_2$ so that the following holds:
Let $\mathcal C$ be any concept hierarchy satisfying the limited-overlap property.  Assume that $B \subseteq C_0$ is presented at time $t$.  Then:
\begin{enumerate}
    \item If $c \in supp_{r_2}(B)$, then with probability at least $1 - \delta$, $rep(c)$ fires at time $t + level(c)$.
    \item If $c \notin supp_{r_1}(B)$ then with probability at least $1 - \delta$, $rep(c)$ does not fire at time $t + level(c)$.
\end{enumerate}

In order to determine appropriate values for $r_1$ and $r_2$,
we start by considering the given sigmoid function.
We determine real values $b_1$ and $b_2$, $-\infty < b_1 < b_2 < \infty$, that guarantee the following:
\begin{enumerate}
    \item If the adjusted potential $x$ is $\geq b_2$, then the probability $p(x)$ of firing is $\geq 1 - \delta'$, and
    \item If the adjusted potential $x$ is $< b_1$, then the probability $p(x)$ of firing is $\leq \delta'$.
\end{enumerate}
Here, we take $\delta'$ to be a small fraction of the target failure probability $\delta$, namely, $\delta' = \frac{\delta}{k^{\lmax+1}}$.
We choose $b_1$ such that $p(b_1+bias) = \frac{1}{1+ e^{-(b_1+bias)/\lambda}} = \delta'$
and $b_2$ such that
$p(b_2+bias) = \frac{1}{1+ e^{-(b_2+bias)/\lambda}} = 1 - \delta'$.
In other words,
$b_1 = \lambda \ log(\frac{1-\delta'}{\delta'}) - bias$, and
$b_2 = \lambda \ log(\frac{\delta'}{1-\delta'}) - bias$.

Next, we compute values for $r_1$ and $r_2$ based on the values of $b_1$ and $b_2$.
The values of $b_1$ and $b_2$ are adjusted potentials, whereas $r_1$ and $r_2$ are fractions of the population of children.
To translate, we use $r_1 = \frac{b_1 + bias}{k}$ and $r_2 = \frac{b_2 + bias}{k}$.
This makes sense because having $r_2 k$ children firing yields a potential of $r_2 k$ and an adjusted potential of $r_2 k - bias = b_2$,
and not having $r_1 k$ children firing means that the potential is strictly less than $r_1 k$ and the adjusted potential is strictly less than $r_1 k - bias = b_1$.

Note that the requirements on $r_1$ and $r_2$ impose some constraints on the value of $bias$.
Namely, since we require $r_1 \geq 0$, we must have $b_1 + bias \geq 0$.
And since we require $r_2 \leq 1$, we must have $b_2+bias \leq k$.
Within these constraints, different values of $bias$ will yield different values of $r_1$ and $r_2$. 

With these definitions, we can prove:

\begin{theorem}
Let $\mathcal C$ be any concept hierarchy satisfying the limited-overlap property. 
Let $\delta$ be a small target failure probability.
Let $r_1$ and $r_2$ be determined as described above.

Assume that $B \subseteq C_0$ is presented at time $t$.  Then:
\begin{enumerate}
    \item If $c \in supp_{r_2}(B)$, then with probability at least $1 - \delta$, $rep(c)$ fires at time $t + level(c)$.
    \item If $c \notin supp_{r_1}(B)$ then with probability at least $1 - \delta$, $rep(c)$ does not fire at time $t + level(c)$.
\end{enumerate}
\end{theorem}

\begin{proof}
(\emph{Sketch:})
Fix any concept $c$ in $C$, the set of concepts in $\mathcal C$.
Note that $c$ has at most $k^{\lmax+1}$ descendants in $C$.

For Property 1, suppose that $c \in supp_{r_2}(B)$.  
Then for every descendant $c'$ of $c$ with $level(c)\geq 1$ and $c' \in supp_{r_2}(B)$, $rep(c')$ fires at time $t+level(c')$ with probability at least $1 - \delta'$, assuming that for each of its children $c'' \in supp_{r_2}(B)$, $rep(c'')$ fires at time $t+level(c'')$.
Using a union bound for all such $c'$, we obtain that, with probability at least 
$1 - k^{\lmax+1} \ \delta' = 1 - \delta$, $rep(c)$ fires at time $t+level(c)$.

In a bit more detail, for each descendant $c'$ of $c$ with $level(c') \geq 1$ and $c' \in supp_{r_2}(B)$, let $S_{c'}$ denote the set of executions in which $rep(c')$ does not fire at time $t + level(c')$, but for each of its children $c'' \in supp_{r_2}(B)$, $rep(c'')$ fires at time $t+level(c'')$.  Then $\bigcup_{c'} S_{c'}$ includes all of the executions in which $rep(c)$ does not fire at time $level(c)$. 

Moreover, we claim that the probability of each $S_{c'}$ is at most $\delta'$:  The fact that $c' \in supp_{r_2}(B)$ imply that at least $r_2 k$ children $c''$ are in $supp_{r_2}(B)$.  Since we assume that all of these $rep(c')$ fire at time $t + level(c'')$, this implies that the potential incoming to $c'$ for time $t + level(c')$ is at least $r_2 k$, and the adjusted potential is at least $r_2 k - bias = b_2$.  Then the first property of $b_2$ yields probability $\leq \delta'$ of $c'$ not firing at time $t + level(c')$.
%[[[Still more carefully, this should talk about each individual length $level(c') - 1$ execution $\alpha$ that ends with all of the $rep(c'')$ firing.  For each such $\alpha$, bound the probability of $rep(c')$ not firing after $\alpha$ by $\delta'$. Then use total probability to get the bound for the set of all such $\alpha$. 

For Property 2, suppose that $c \notin supp_{r_1}(B)$.  
Then for every descendant $c' \notin supp_{r_1}(B)$, $rep(c')$ does not fire at time $t+level(c')$, with probability at least $1 - \delta'$, assuming that for each of its children $c'' \notin supp_{r_1}(B)$, $rep(c'')$ does not fire at time $t+level(c'')$.
%There are at most $k^{\lmax+1}$ descendants of $c$ that are not in $supp_{r_1}(B)$, so 
Using a union bound for all $c' \notin supp_{r_1}(B)$, we obtain that, with probability at least 
$1 - k^{\lmax+1} \ \delta' = 1 - \delta$, does not fire at time $t+level(c)$.

In a bit more detail, for each descendant $c'$ of $c$ with $level(c') \geq 1$ and $c' \notin supp_{r_1}(B)$, let $S_{c'}$ denote the set of executions in which $rep(c')$ fires at time $t + level(c')$, but for each of its children $c'' \notin supp_{r_1}(B)$, $rep(c'')$ does not fire at time $t+level(c'')$.  Then $\bigcup_{c'} S_{c'}$ includes all of the executions in which $rep(c)$ fires at time $level(c)$, and the probability for each $S_{c'}$ is at most $\delta'$; the argument is similar to that for Property 1.

\end{proof}

This has been only a sketch of how to analyze stochastic firing, in the simple case of feed-forward networks and exact weights.
We leave the complete details of this case, as well as extensions to include approximate weights and networks with feedback, for future work.

\end{document}